\documentclass[twoside,leqno,twocolumn]{article}
\usepackage{amssymb}
\newcommand{\R}{\mathbb{R}}

\usepackage[letterpaper]{geometry}

\usepackage{siamproceedings}

\usepackage[T1]{fontenc}
\usepackage{amsfonts}
\usepackage{graphicx}
\usepackage{epstopdf}
\usepackage{enumitem}
\usepackage{algorithmic}
\ifpdf
  \DeclareGraphicsExtensions{.eps,.pdf,.png,.jpg}
\else
  \DeclareGraphicsExtensions{.eps}
\fi

\newsiamremark{remark}{Remark}
\newsiamremark{hypothesis}{Hypothesis}
\crefname{hypothesis}{Hypothesis}{Hypotheses}
\newsiamthm{claim}{Claim}

\usepackage{amsopn}

\newcommand{\best}[1]{\textbf{#1}}

\def\D{\mathcal{D}}
\def\B{\mathcal{B}}
\def\S{\mathcal{S}}
\def\Q{\mathcal{Q}}

\newcommand{\our}{SeBA}
\newcommand{\full}{Separated-at-Birth Alignment}

\usepackage{subcaption}
\usepackage{amsfonts}
\usepackage{amssymb}
\usepackage{booktabs}
\usepackage{multirow} 
\usepackage{rotating}

\usepackage{array}
\usepackage{siunitx}
\begin{document}

\newcommand\relatedversion{}

\title{\Large \our{}: Semi-supervised few-shot learning via\\ \full{} for tabular data\relatedversion}

\author{ Kacper Jurek\thanks{Equal contribution.} \thanks{Jagiellonian University, Faculty of Mathematics and Computer Science} \thanks{Jagiellonian University, Doctoral School of Exact and Natural Sciences} \and Wojciech Batko\footnotemark[1] \footnotemark[2] \and Marek Śmieja\footnotemark[2] \and Marcin Przewięźlikowski\footnotemark[2] \thanks{NASK National Research Institute} \thanks{Corresponding author: \email{marcin.przewiezlikowski@uj.edu.pl}.} }

\date{}

\maketitle

\fancyfoot[R]{\scriptsize{Copyright \textcopyright\ 2026 by SIAM\\
Unauthorized reproduction of this article is prohibited}}

\begin{abstract}

Learning from scarce labeled data with a larger pool of unlabeled samples, known as semi-supervised few-shot learning (SS-FSL), remains critical for applications involving tabular data in domains like medicine, finance, and science.
The existing SS-FSL methods often rely on self-supervised learning (SSL) frameworks developed for vision or language, which assume the availability of a natural form of data augmentations. For tabular data, defining meaningful augmentations is non-trivial and can easily distort semantics, limiting the effectiveness of conventional SSL. In this work, we rethink SSL for tabular data and propose \full{} (\our{}), a joint-embedding framework for SS-FSL that eliminates the dependence on augmentations. 
Our core idea is to separate the data into two independent, but complementary views and align the representations of one view to mirror the nearest-neighbor correspondence of the data in the second view. 
Our experimental evaluation supported by a theoretical analysis justifies that \our{} generates an output space, which improves the feature-label relationship.
An experimental study conducted in various benchmark datasets demonstrates that \our{} achieves the state-of-the-art performance in the majority of cases, opening a new avenue for SS-FSL paradigm in the domain of tabular data. Code is available at \url{https://github.com/kacper3615/SeBA}

\end{abstract}

\section{Introduction}

Learning with a limited amount of labeled data remains a fundamental challenge in machine learning and data analysis. 
Although collecting additional annotations is costly, access to raw unlabeled data is often inexpensive.
This imbalance motivates the practical setting of semi-supervised few-shot learning (SS-FSL), where classification must be performed with scarce labeled data and a large pool of unlabeled samples (see~\Cref{fig:setting_ssfsl}).
Applications in disease diagnosis~\cite{shailaja2018machine}, credit risk prediction~\cite{clements2020sequentialdeeplearningcredit}, and cognitive sciences~\cite{grabowska2025individual} highlight the need for approaches tailored to tabular data. Despite its importance, tabular SS-FSL has been underexplored compared to FSL methods for computer vision (CV) and natural language processing (NLP)~\cite{finn2017model,snell2017prototypical,sendera2023hypershot,przewikezlikowski2022hypermaml,gpt,min2022noisychannellanguagemodel}.

The tabular modality poses unique challenges for typical SS-FSL methods, which rely on pretraining with unlabeled data followed by fine-tuning on a few labeled examples (see \Cref{fig:problem}).
State-of-the-art pretraining approaches often use self-supervised learning (SSL), which encourages models to produce similar representations for semantically related \emph{positive pairs} while avoiding collapse to trivial solutions~\cite{wang2020understanding}.
Such pairs are usually created by sampling multiple augmentations of the same data point.
In CV, augmentations are straightforward: image transformations such as cropping, rotation, or color jittering yield valid semantically consistent samples.
However, for tabular data, there is no general way to define proper augmentations.
Poorly chosen transformations, such as zero masking, Gaussian noise, or sampling features from marginal distribution, can distort semantics or even generate out-of-distribution samples (see~\Cref{fig:augmentation_problem}), ultimately undermining the effectiveness of SSL.
Consequently, recent work has turned to alternative pretraining strategies for tabular representation learning --  such as cluster detection with prototypical networks~\cite{nam2023stunt}, or diffusion-based methods~\cite{liu2024d2r2} -- largely sidelining SSL.

\begin{figure*}[t]
    \centering
    \begin{subfigure}[t]{0.31\textwidth}
        \includegraphics[width=\linewidth]{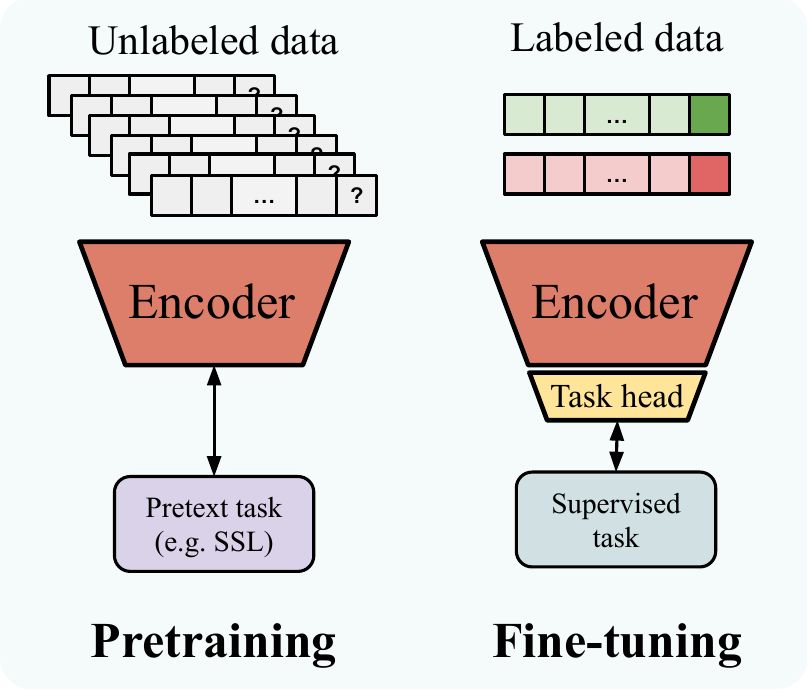}
        \caption{Semi-supervised Few-shot Learning (SS-FSL) setup. The model is pretrained on a large pool of unlabeled data and next fine-tuned on a few labeled examples.}
        \label{fig:setting_ssfsl}
    \end{subfigure}
    \hfill
    \begin{subfigure}[t]{0.66\textwidth}
    \includegraphics[width=\linewidth]{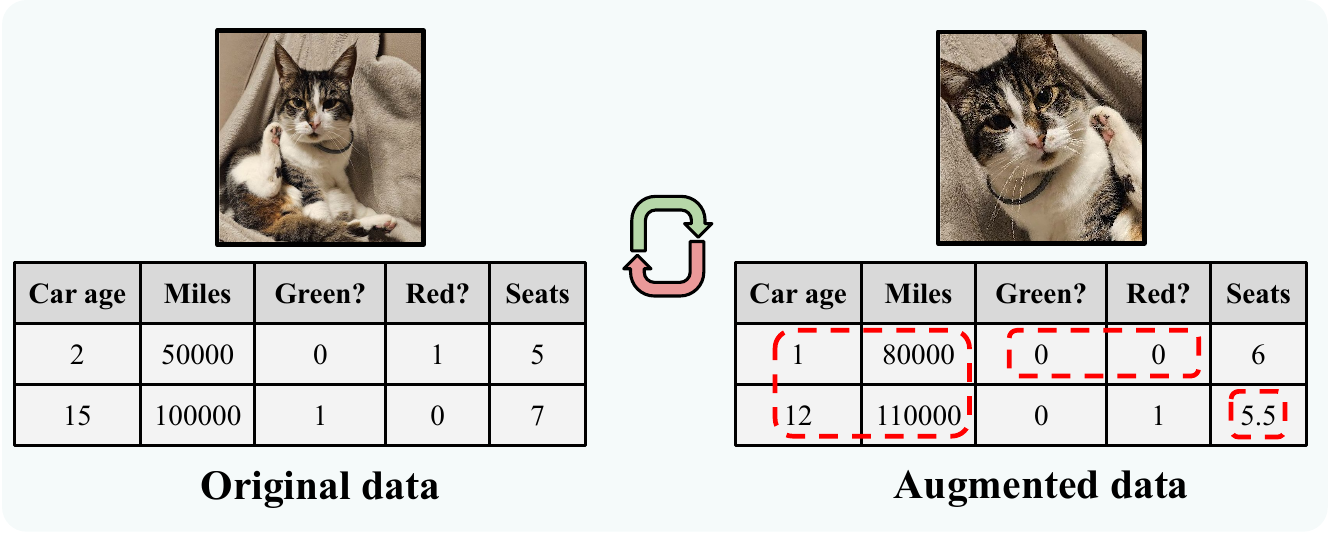}
        \caption{While semantic-preserving augmentations are straightforward to define for modalities such as images, they must be designed much more carefully for tabular data. Improperly designed augmentations can generate samples from outside the data manifold (decreasing car age, but increasing mileage), obfuscate the categorical values (neither option marked as true), or assign incorrect values (number of car seats must be an integer).}
        \label{fig:augmentation_problem}
    \end{subfigure}
    \caption{
    Typical Semi-supervised Few-shot Learning (SS-FSL) approaches \textbf{(a)} pretrain their representations on large pools of unlabeled data, usually via Self-supervised Learning (SSL).
    In the case of tabular data, the state-of-the-art augmentation-based SSL approaches cannot be directly applied, due to challenges with defining proper augmentations \textbf{(b)}. 
    In this work, we introduce \full{}, which removes the need for augmentations and improves the tabular representations for SS-FSL (see~\Cref{fig:seba_architecture}).
    }
    \label{fig:problem}
\end{figure*}

In this paper, we rethink SSL for tabular data and show that, with carefully designed positive pairs, it yields significantly stronger tabular representations than previously assumed.
Instead of aligning the representations of positive pairs created via augmentations, we introduce \full{} (\our{}), illustrated in \Cref{fig:seba_architecture}. 
\our{} projects data into two complementary 
subspaces: feature and target views. 
The model is then pretrained by identifying nearest-neighbor correspondences in the target view, using only the information encoded in the feature view.
This replaces the problematic reliance on augmentations with a nearest-neighbor graph.
To properly handle mixed data types, we employ a type-aware separation scheme that accounts for both categorical and numerical features, ensuring that the resulting projections remain semantically meaningful.

\our{} requires far less dataset-specific knowledge than hand-crafting augmentations, making it practical and easy to apply.
Moreover, the model pretrained by \our{} is lightweight and thus less prone to overfitting for small datasets.
The experimental results clearly show that \our{} generalizes effectively in a wide variety of tabular datasets, achieving impressive results in few-shot classification, especially on the challenging, high-dimensional benchmarks.

\textbf{Our contributions can be summarized as follows:}

\begin{itemize}%
    \item We introduce \full{} (\our{}), a novel self-supervised pretraining approach for tabular Few-shot Learning that replaces augmentation-based positive pair construction with a nearest-neighbor graph, removing the need for hand-crafted augmentations.
    
    \item We provide an in-depth empirical evaluation that confirms that our pretraining generates an output space, which improves the feature-label relationship, see \Cref{fig:10nn_consistency}. This is supported by a theoretical analysis for Gaussian data, which justifies that \our{} produces meaningful positive pairs. %

    \item We experimentally verify that \our{} generalizes across a wide variety of tabular datasets, achieving strong performance in few-shot classification, see~\Cref{sec:main_results}.
\end{itemize}

\section{Related work}

\paragraph{Self-supervised learning for tabular data}

One of the first approaches to self-supervised learning for tabular data, VIME \cite{yoon2020vime}, creates a novel pretext task of estimating mask vectors from corrupted tabular data in addition to the reconstruction pretext task for self-supervised learning. \cite{bahri2021scarf} propose SCARF, a contrastive learning method in which different views of a sample are obtained by corrupting a random subset of features. In both methods, data corruption is implemented as sampling from empirical distribution on masked features. \cite{ucar2021subtab} propose SubTab that divides the input features into multiple subsets to perform a pretext task close to mask reconstruction and contrastive learning simultaneously. \cite{sui2023self} propose an augmentation-free method that simultaneously reconstructs multiple randomly generated data projection functions to generate a data representation. T-JEPA \cite{thimoniert} is also an augmentation-free method that is trained by predicting the latent representation of a subset of features from the latent representation of a different subset within the same sample. PTaRL has investigated prototype-based representation learning for tabular data \cite{ye2024ptarl}, which first constructs a prototype-based projection space (P-Space) and then learns the disentangled representation around the global data prototypes. 
Most of these methods are applied as pretraining techniques, where several deep classifiers are trained on the obtained representation, which is not well suited to SS-FSL problem considered here.

\paragraph{Few-shot learning for tabular data}

Some recent research claims that SSL methods cannot create representations for tabular data that can be fine-tuned to target tasks with a limited number of labels \cite{nam2023stunt, liu2024d2r2}. As a remedy, \cite{nam2023stunt} propose STUNT, a meta-learning method that pretrains a ProtoNet on self-generated tasks using only unlabeled data. The authors show that given only a few labeled samples per class, STUNT outperforms other methods. \cite{liu2024d2r2} propose D2R2 to extract representations by training a conditional diffusion model and aligning the distances from various random projection spaces.
Although D2R2 reports the SOTA results, they are obtained in a transductive setting, in which the model is tuned according to the information coming from the test set. We are the first who show that inferior performance of SSL methods in a few-shot setting is not inherent to tabular data. The proposed construction of \our{} allows us to construct a representation that is state-of-the-art in tabular few-shot learning.

\paragraph{Transfer learning and other pretraining techniques}

Other works, such as XTab \cite{zhu2023xtab}, TransTab \cite{wang2022transtab}, or UniTabE \cite{yang2023unitabe}, propose a cross-table pretraining approach for building representation. A seminal work of \cite{hollmanntabpfn} introducing TabPFN demonstrated how to pretrain a transformer-based architecture on synthetically generated data. Once the model is pretrained, it can be applied in a zero-shot setting to new data or further fine-tuned on target task \cite{breejen2024fine}. 
In contrast to the transformer-based methods mentioned above, \our{} uses a lightweight MLP and a single linear layer as the task-specific head, which can be fine-tuned effectively using only a few labeled samples.%

\section{Method}

\begin{figure*}[t]
    \centering
    \includegraphics[width=0.8\linewidth]{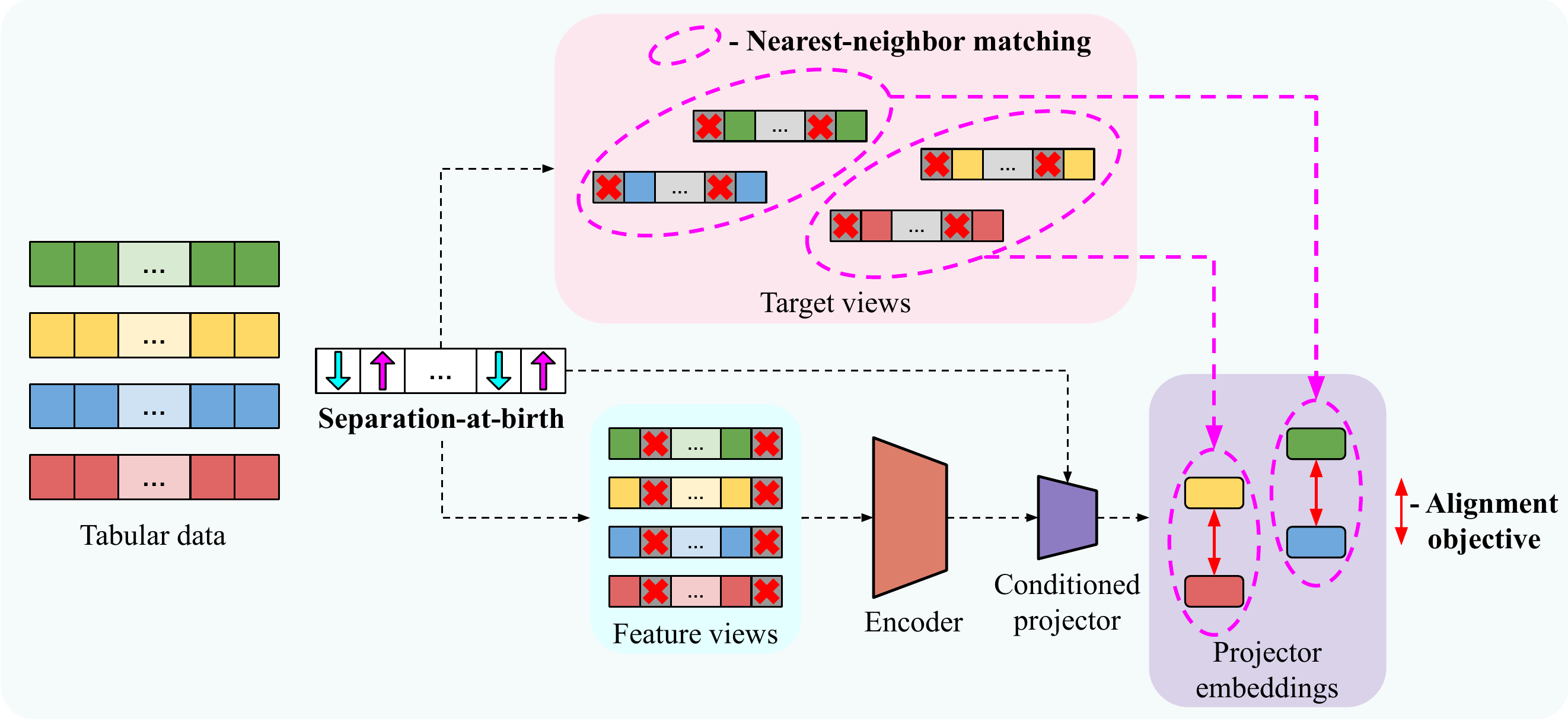}
    \caption{
    Representation learning via \full{} (\our{}).
    In each minibatch, we separate the columns of tabular data "at-birth" into two complementary and independent subsets, which define target and feature views.
    Instead of augmentation, semantically-related positive pairs for a pretraining contrastive task are defined using the nearest neighbor relation in their target view (upper side). The encoder is trained to create the representation of the feature view, which aligns the positive pairs defined in the target view (bottom side). To allow the encoder to create general data representation, \our{} uses a conditioned projector to build a task-specific representation for every separation mask.
    }
    \label{fig:seba_architecture}
\end{figure*}

\paragraph{Problem statement and scope.} We consider a semi-supervised few-shot learning problem (SS-FSL), illustrated in~\Cref{fig:setting_ssfsl}. 
In the training stage, we have access to a small labeled support set $\S = \{(s_{i},y_{i})\}_{i=1}^{N_s}$ and a large portion of unlabeled data $\D=\{(u_i)\}_{i=1}^{N_u}$, where $s_i, u_i \in \R^D$ are feature vectors, $y_i \in \{1,2,\ldots,N\}$ denotes the class label and $N$ is the number of classes. We assume that $N_u \gg N_s$. Furthermore, in the $N$-way $K$-shot setting, each of the $N$ classes in the support set is represented by exactly $K$ labeled samples.
Our goal is to train a model $h_\theta: \R^D \to \{1,2,\ldots,N\}$ for the classification task defined by the support data.
In a typical inductive setting considered here, 
we evaluate the model on the query set $\Q = \{q_i\}_{i=1}^{N_q}$ of previously unseen data.

The typical SS-FSL training divides the training of $h_\theta$ into two phases: \textbf{(i)} unsupervised pretraining of the encoder $f:\R^D \to \R^E$, and \textbf{(ii)} supervised training of a classifier $c: \R^E \to \R^N$ on $\S$ on top of the representation created by $f$, where $E$ is the embedding shape of the encoder~\cite{nam2023stunt,liu2024d2r2}. Therefore, the final model is given by $h_\theta = c \circ f$.

In this work, we follow the standard protocol for training the few-shot classifier and focus primarily on the pretraining phase, where we introduce a novel approach called \full{} (\our{}), described below.

\subsection{\full{} (\our{})}

\paragraph{Overview.}

The design of \our{} follows Self-supervised Joint-Embedding Architectures (JEAs), which learn through aligning semantically-related positive pairs of data in the representation space~\cite{chen2020simple,he2020momentum} and pushing away the unrelated ones. 
Unlike conventional JEAs, \our{} does not rely on sampling multiple data augmentations for the construction of positive pairs, which is problematic for tabular data.
Instead, in every minibatch, \our{} separates the tabular records "at birth" into two random complementary views, which we denote as feature and target views. 
The model is trained to align the data representations of feature views according to the similarity graph induced by the target views, therefore, learning semantically meaningful correspondences without relying on augmentations.
We outline the schema of \our{} in \Cref{fig:seba_architecture} and describe its components in detail below.

\paragraph{Separating the data at-birth.}

Let $\B$ be a minibatch of (unlabeled) data points, and let $m \in \{0,1\}^D$ be a binary mask vector sampled once per batch that defines features included in each view. 
The proportion of $1$-s in $m$ is controlled by a \emph{separation ratio}.
For every $x \in \B$, we create a feature view:
\begin{equation} \label{eq:feature}
x_f = x \odot (1-m) %
\end{equation}
and a target view:
\begin{equation}
x_t = x \odot m, 
\end{equation}
where $\odot$ is an element-wise multiplication, and $x_t, x_f \in \R^D$.
As such, the target and feature views are complementary and independent. 

\paragraph{Target similarity graph.}
We use target views to define the positive data pairs with respect to the sampled mask $m$. 
For each sample $x$ in the batch, we identify its nearest neighbor in terms of the target views:
\begin{equation} 
\label{eq:near}
x' = \arg \min_{a \in \B \setminus \{x\}} d(x_t,a_t).    
\end{equation}
In other words, $m$ defines positive pairs $(x,x')$ based on the nearest-neighbor graph defined in the target view.
In Appendix~\ref{app:theory}, we present the theoretical justification that our approach produces positive pairs, which are meaningful for solving a target task. Given Gaussian classes, we consider the probability of the mismatch event that a sample and its nearest neighbor calculated in a target view come from different classes. We derive that the expected value of this probability decays exponentially as the dimensionality of the target view and the separation between Gaussians increase. In consequence, we are allowed to perform nearest neighbor search in target view to produce meaningful positive pairs without the risk that the nearest neighbor will have a different class.

\paragraph{Alignment objective.}
Finally, we train the encoder to align the feature-view representations to match the nearest-neighbor relation defined in the target view.
For this purpose, we first construct the encoder representations of the feature views:
\begin{equation}
    h = f(x_f); h'=f(x'_f)
\end{equation}
Observe that the feature views $x_f$ do not contain unequivocal information about the data separation scheme $m$ and, as such, may not be sufficient to solve the alignment task on their own.
To address this problem, \our{} incorporates a conditioned projector $\pi: \R^E \times \R^D \to \R^P$, where $P$ is the embedding shape of $\pi$~\cite{przewiezlikowski2024augmentation,bordes2022guillotine}. 
The projector combines the general feature view representation of the encoder and information about how the data were separated (i.e. the mask vector $m$). 
The projector transforms the encoder representation into the task-specific latent space:
\begin{equation}
    z = \pi(h, m); z' = \pi(h', m),
\end{equation}
in which we optimize the alignment objective.
The objective takes form of the InfoNCE loss, which pulls together the positive representation pairs, and pushes away the unrelated ones~\cite{oord2018representation}:
\begin{equation}
    \label{eq:infonce}
    \mathcal{L}(x) = -\log \dfrac{
        \exp(d(z, z'))
    }{
        \Sigma_{a \in \B, a \neq x } \exp \biggl( d \Bigl( z, \pi \bigl(f(a_f), m \bigr) \Bigr) \biggr)
    }
\end{equation}
Because \our{} trains on numerous separation schemes (multiple mask vectors), the encoder adapts repeatedly to different target matching objectives. This exposure yields robust representations that generalize well to downstream tasks.

\subsection{Pretraining details}

The proposed pretraining benefits from several design choices, discussed below, and ablated in~\Cref{sec:ablation}.

\paragraph{Type-aware feature preprocessing and separation.} We normalize the numerical columns by removing the mean and scaling to unit variance, and encode the categorical data using one-hot vectors.
A straightforward separation applied to that representation could result in splitting a single categorical variable between the two views, even though all of its coordinates represent the same variable. 
The proper way to take into account the specific nature of the categorical variable, is to define the separation mask $m$ in the original representation (before the one-hot encoding), which prevents splitting a single category into two views. 

\paragraph{Missing data encoding.} 

When separating the data into feature and target views, the question arises of what to replace the separated features with.
Although this issue could be eliminated in a transformer-based architecture, it must be addressed with our chosen lightweight MLP encoder 
which is more suitable for small datasets~\cite{thimoniert}. 
In contrast to previous works, which commonly sampled masked features from empirical distribution \cite{nam2023stunt}, or trained a dedicated missing data imputation module ~\cite{smieja2018processing},
we find that zero imputation works best in our case.

\paragraph{Separation ratio.} To allow the encoder to create a general representation that suits multiple tasks, we could iterate over all possible mask vectors when separating the data into views. However, feeding the network with tabular masked records with variable mask size could hurt its performance \cite{yi2020not}. 
Consequently, during pretraining, we constrain the selection of the mask by fixing the separation ratio to a constant value. 
Choosing a single, most appropriate value of the separation ratio is challenging for two reasons: (i) tabular datasets are highly diverse and the amount of sufficiently descriptive features varies from dataset to dataset, and (ii), in SS-FSL settings with an extremely low amount of labeled data, we cannot allocate a portion of labeled samples for validation. 
Instead, we propose to generate multiple representations, each trained using a different separation ratio $r \in [0.1, 0.2, 0.3, 0.4, 0.5]$.
During inference, we ensemble their predictions which helps to avoid the aforementioned problems.
In practice, we find that while individual models trained with a single separation ratio achieve competitive results on different datasets, this strategy yields the most overall robust results.

\section{Experiments}

\subsection{Experimental setup}

We follow the benchmark proposed by~\cite{nam2023stunt} and then developed by~\cite{liu2024d2r2} verifying the performance of the models in a few-shot learning scenario.

\paragraph{Datasets preparation.}

Following~\cite{nam2023stunt}%
, we select 8 datasets from the OpenML-CC18 benchmark~\cite{asuncion2007uci,bischl2017openml}, on which we evaluate the performance of \our{}. We summarize the datasets in \Cref{sec:datasets}.
Following~\cite{nam2023stunt}, all datasets are randomly divided into train and test sets in a ratio of 5:1. Training data is treated as unlabeled and used for model pretraining. In addition, 10\% of the unlabeled training data is used for validation. Once the model is pretrained, it is fine-tuned on the support set and evaluated on the query set. The support and query samples are randomly sampled from the test set. In the $K$-shot $N$-way setting, the support set contains $K$ examples of each of $N$ classes. 
 We consider 1-, 5-, and 10-shot settings.

\paragraph{Setup of \our{}.} 

We pretrain the encoder and projector of \our{} for 10000 epochs, and perform early stopping when the value of the objective function, measured on the validation set, stops decreasing for 100 epochs. 
Following~\cite{nam2023stunt}, the encoder is a 2-layer MLP with a hidden dimension of 1024, and the projector is also a 2-layer network with the same hidden dimension and an embedding dimension of 256, a common choice in contrastive learning~\cite{chen2021empirical}.
In the downstream task stage, we freeze the encoder and train a classification head at the top of the encoder representation using the support set. For 5- and 10-shot, we use linear probing, while for 1-shot setting, we assign query samples to the closest, in terms of cosine distance, class prototypes obtained from the support set. 

In summary, for each dataset, we repeat the following procedure with 100 random seeds: First, we split the data and pretrain the encoder. Next, we sample the support and query sets from the labeled test data 100 times and average the accuracy metrics over all splits and all selections of the support/query sets. Therefore, every result is reported by averaging 10000 runs in total.

\subsection{Few-shot classification}
\label{sec:main_results}

\begin{table}[t]
\centering
\caption{Evaluation in terms of 1-shot classification accuracy.}
\footnotesize
\label{tab:fewshot_main_results_1shot}
\setlength{\tabcolsep}{2pt}
\resizebox{\linewidth}{!}{
\begin{tabular}{l c c c c c c c c c c}
\toprule
\textbf{Method} &
\multicolumn{1}{c}{CMC} &
\multicolumn{1}{c}{DIA} &
\multicolumn{1}{c}{DNA} &
\multicolumn{1}{c}{INC} &
\multicolumn{1}{c}{KAR} &
\multicolumn{1}{c}{OPT} &
\multicolumn{1}{c}{PIX} &
\multicolumn{1}{c}{SEM} &
\\
\midrule
\midrule
CatBoost & 36.03 & 56.74 & 39.15 & 57.55 & 53.24 & 58.30 & 54.74 & 43.21 \\
kNN & 35.39 & 58.50 & 42.20 & 51.45 & 54.61 & 65.60 & 60.79 & 44.35 &  \\
TabPFN & 35.37 & 53.35 & 41.83 & 49.30 & 46.02 & 55.74 & 23.79 & 28.01 &  \\
SubTab & 36.23 & 58.22 & 46.98 & 62.45 & 50.22 & 62.01 & 60.34 & 39.99 & \\
VIME & 35.90 & 58.99 & 51.23 & 61.82 & 59.81 & 69.26 & 63.28 & 46.99 &  \\
Scarf & 35.39 & 55.64 & 57.86 & 57.94 & 60.96 & 63.31 & 63.93 & 29.39 \\
T-JEPA & 34.28 & 50.44 & 36.14 &  48.97 &  30.22 & 39.84 & 39.89 & 25.47  \\
PseudoL & 34.97& 57.03 & 44.26 & 60.52 & 49.44 & 61.50 & 56.12 & \text{41.42} \\
MeanT & 35.58 & 58.05 &  46.58 &  60.63 & 54.57 & 66.10 & 61.02 & \text{43.56}  \\
SAINT & 35.41 & 59.65 &  36.88 &  54.69 & 40.13 & 52.92 & 24.25 & 32.52  \\
STUNT & 37.10 & 61.08 & 66.20 & 63.52 & 71.20 & 76.94 & 79.05 & 55.91 &
\\
D2R2-c & \best{40.81} & 60.10 & 61.29 & \best{72.85} & 61.45 & 77.41 & 61.45 & 34.26 & 
\\
\midrule
\our{} & 36.76 & \best{61.14} & \best{66.79} & 62.89 & \best{76.40} & \best{78.94} & \best{83.06} & \best{61.11}  &\\ 
\bottomrule
\end{tabular}
}
\end{table}

\begin{table}[t]
\centering
\caption{Evaluation in terms of 5-shot classification accuracy.}
\footnotesize
\label{tab:fewshot_main_results_5shot}
\setlength{\tabcolsep}{2pt}
\resizebox{\linewidth}{!}{
\begin{tabular}{l c c c c c c c c}
\toprule
\textbf{Method} &
\multicolumn{1}{c}{CMC} &
\multicolumn{1}{c}{DIA} &
\multicolumn{1}{c}{DNA} &
\multicolumn{1}{c}{INC} &
\multicolumn{1}{c}{KAR} &
\multicolumn{1}{c}{OPT} &
\multicolumn{1}{c}{PIX} &
\multicolumn{1}{c}{SEM} \\
\midrule
\midrule
CatBoost & 39.89 & 64.51 & 60.20 & 67.99 & 77.94 & 83.07 & 83.38 & 68.69  \\
kNN & 37.65 & 65.61 & 61.16 & 62.19 & 80.08 & 84.16 & 84.75 & 68.33 \\
TabPFN & 38.31 & 64.06 & 52.72 & 64.11 & 76.59 & 81.68 & 62.41 & 56.20  \\
SubTab & 39.81 & 68.26 & 62.49 & 72.14 & 70.88 & 83.27 & 80.41 & 59.87  \\
VIME & 39.83 & 67.64 & 71.29 & 72.19 & 19.42 & 83.21 & 85.24 & 68.45 \\
Scarf & 37.75 & 68.66 & 62.75 & 66.09 & 69.96 & 85.67 & 81.32 & 35.20 \\
T-JEPA & 37.25 & 64.06 & 43.04 & 62.13 & 60.07 & 71.93 & 73.66 & 50.56  \\
PseudoL & 37.49 & 64.46 & 60.06 & 66.26 & 78.60 & 83.71 & 82.94 & 82.94\\
MeanT & 37.73 &  65.45  & 61.47 & 67.05 & 81.08 & 86.66 & 85.24 & \text{69.67} \\
SAINT & 35.41 & 66.94 &  46.34 &  60.01 & 71.07 & 79.53 & 40.35 & 62.42 \\
STUNT & 40.40 & \best{69.88} & 79.18 & 72.69 & 85.45 & 88.42 & 89.08 & 71.54   
\\
D2R2-c & \best{43.39} & 68.69 & \best{81.39} & \best{73.34} & 79.49 & 87.12 & 82.22 & 60.16  
 \\
\midrule
\our{} & 42.85 & 69.54 & 79.86 & 71.28 & \best{87.59} & \best{90.11} & \best{91.88} & \best{79.41}  \\
\bottomrule
\end{tabular}
}
\end{table}

\begin{table}[t]
\centering
\caption{Evaluation in terms of 10-shot classification accuracy.}
\footnotesize
\label{tab:fewshot_main_results_10shot}
\setlength{\tabcolsep}{2pt}
\resizebox{\linewidth}{!}{
\begin{tabular}{l c c c c c c c c}
\toprule
\textbf{Method} &
\multicolumn{1}{c}{CMC} &
\multicolumn{1}{c}{DIA} &
\multicolumn{1}{c}{DNA} &
\multicolumn{1}{c}{INC} &
\multicolumn{1}{c}{KAR} &
\multicolumn{1}{c}{OPT} &
\multicolumn{1}{c}{PIX} &
\multicolumn{1}{c}{SEM} \\
\midrule
\midrule
CatBoost & 38.20 & 67.26 & 70.56 & 70.75 & 81.40 & 85.54 & 81.78 & 72.52 \\
kNN & 39.93 & 63.35 & 57.94 & 66.42 & 81.84 & 85.34 & 84.04 & 73.86 \\
TabPFN & 43.18 & 67.85 & 60.87 & 68.23 & 80.20 & 86.23 & 75.16 & 67.87 \\
SubTab & 42.12 & 67.31 & 64.14 & 72.81 & 85.27 & 85.79 & 86.75 & 70.44 \\
VIME & 39.20 & 65.78 & 62.94 & 69.15 & 76.12 & 83.57 & 45.61 & 70.65 \\
Scarf & 39.16 & 66.47 & 75.35 & 63.52 & 70.59 & 78.11 & 84.11 & 64.14 \\
T-JEPA & 38.17 & 66.75 & 48.70 & 66.56 & 73.51 & 82.00 & 83.56 & 63.68  \\
PseudoL & 42.01 & 65.90 & 64.49 & 63.08 & 86.51 & 89.29 & 43.28 & 77.86 \\
MeanT & 42.04 & 68.00 & 49.40 & 66.99 & 80.49 & 86.84 & 62.24 & 72.20 \\
SAINT & 42.74 & 68.10 & 61.76 & 64.87 & 79.06 & 85.21 & 58.25 & 70.03 \\
STUNT & 42.01 & 72.82 & 80.96 & 74.08 & 86.95 & 89.91 & 89.98 & 74.74 \\
D2R2-c & 38.41 & 71.94 & 81.46 & \best{74.84} & 86.19 & 91.84 & 84.93 & 69.56 \\
\midrule
\our{} & \best{46.30} & \best{73.61} & \best{83.59} & 72.68 & \best{90.88} & \best{92.62} & \best{93.88} & \best{84.11} \\
\bottomrule
\end{tabular}
}
\end{table}

\begin{figure}
    \centering
    \includegraphics[width=\linewidth]{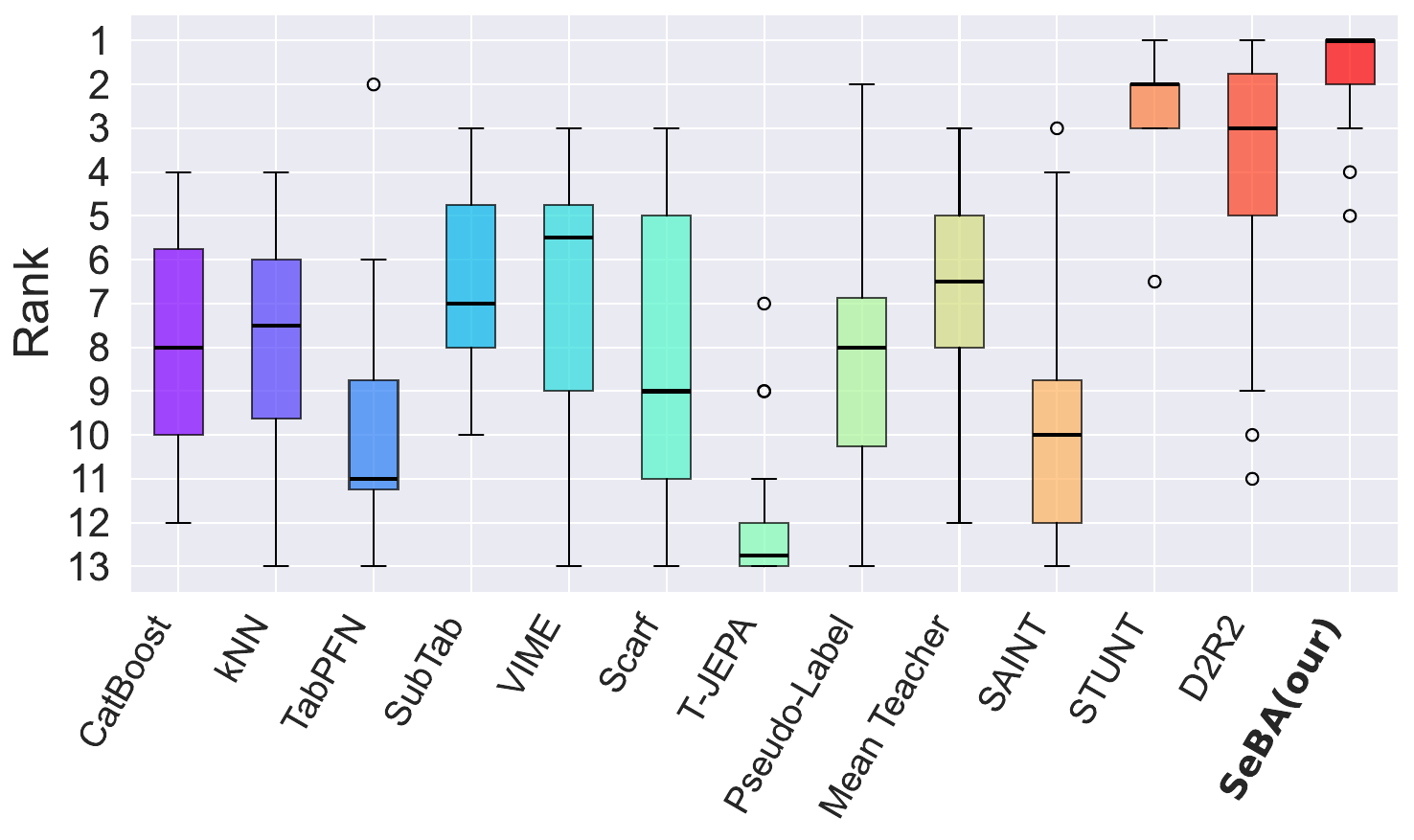}
    \caption{Summary of performance ranks of each analyzed method across few-shot classification tasks. \our{} is, on average, the best-ranking approach.}
    \label{fig:ranks_aggr}
\end{figure}

We evaluate \our{} in terms of its performance in downstream few-shot learning tasks.
We compare our method with the state-of-the-art SS-FSL methods,  STUNT~\cite{nam2023stunt}, and D2R2\footnote{We use an inductive variant of D2R2, i.e. D2R2-c, which uses mean support embeddings as the classifier~\cite{liu2024d2r2}. The default D2R2 uses an instance-wise iterative prototype scheme, additionally using query data for class prototype estimation. This is not consistent with the inductive setting, where queries are unseen during classifier training.}~\cite{liu2024d2r2}, as well as 
10 other baselines which represent the state-of-the-art supervised, self-supervised, and semi-supervised approaches to tabular data (see \Cref{app:methods} for more details). We report the average results of each method under the same evaluation protocol and train-test splits as \our{}.

\begin{figure*}[t]
    \centering
    \begin{subfigure}[t]{0.49\textwidth}
        \centering
        \includegraphics[width=\textwidth]
        {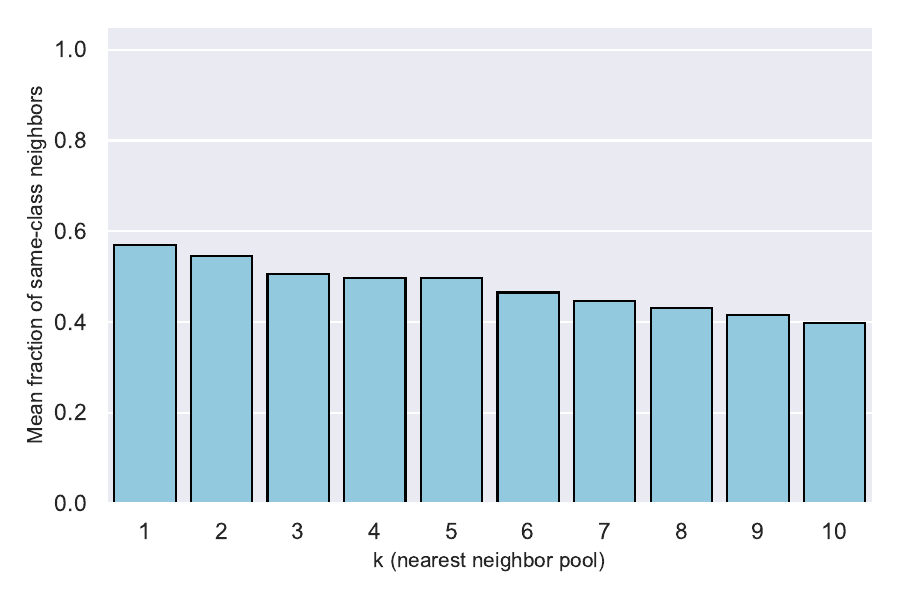}
        \caption{Fraction of same-class samples among $k$ nearest neighbors.}
        \label{fig:knn_fraction}
    \end{subfigure}
    \hfill
    \begin{subfigure}[t]{0.49\textwidth}
        \centering
        \includegraphics[width=\textwidth]{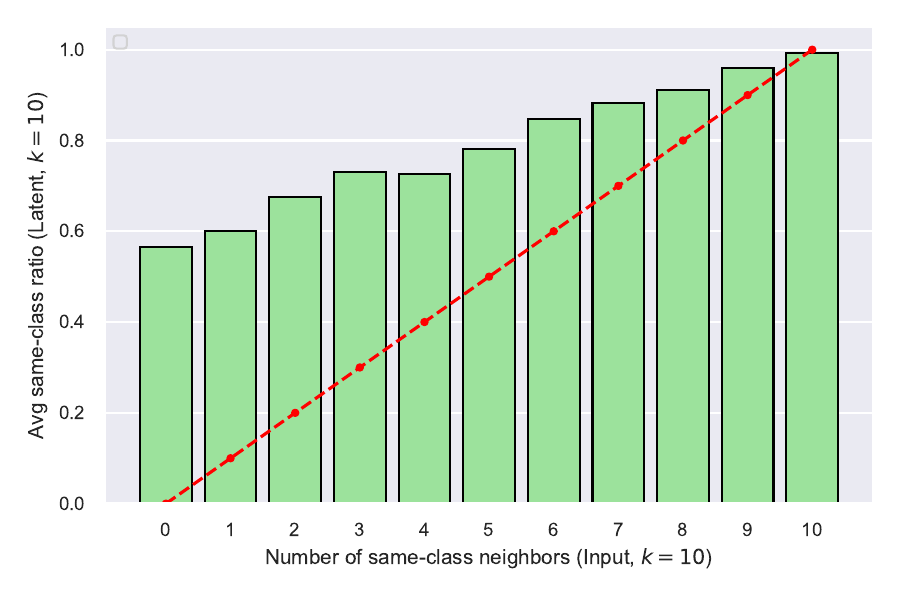}
        \caption{Comparison of 10-NN label consistency between input and latent spaces.}
        \label{fig:10nn_consistency}
    \end{subfigure}
    \caption{Analysis of \our{} on the Optdigits dataset: 
             (a) the fraction of same-class samples among the k nearest neighbors decreases as k increases, 
             and (b) the latent space significantly enhances label consistency compared to the input space baseline.
             }
    \label{fig:neighbor_analysis}
\end{figure*}

We present the results of the evaluation in 1-, 5-, and 10-shot classification in~\Cref{tab:fewshot_main_results_1shot,tab:fewshot_main_results_5shot,tab:fewshot_main_results_10shot}, respectively, and summarize them collectively in~\Cref{fig:ranks_aggr}, and in~\Cref{fig:ranks} separately for each value of $K$.
The efficacy of \our{} increases consistently with the number of support examples, as opposed to approaches like D2R2-c, which exhibit significant variance in quality on datasets like SEM.
\our{} achieves the best accuracy in
17 out of 24
instances and 
is among the three best approaches in 5 out of the remaining 7 instances,
which further confirms its practicality and applicability to a wide range of datasets. 
Notably, in all but one cases, \our{} achieves the best performance on datasets containing purely categorical values (DNA \& PIX). 
Moreover, the performance gap between \our{} and the baseline approaches is especially visible on highly-dimensional sets, such as PIX and SEM.
Furthermore, in ~\Cref{tab:fewshot_main_results_with_std} we present the standard deviations of the results of the three best-performing methods (STUNT, D2R2-c, and \our{}) and find that the results of \our{}  vary less than those of D2R2-c and STUNT, highlighting the greater stability of our method. 
We summarize the average performance of each method in~\Cref{fig:ranks}, from which it is evident that \our{} is generally the best-performing approach.

Beyond the standard datasets used to evaluate Semi-supervised Few-shot Learning  approaches~\cite{nam2023stunt,liu2024d2r2}, we also test \our{} on three additional benchmarks, which exemplify the common challenges in real-world tabular datasets: high data dimensionality and feature sparsity: CNAE, Devnagari, and FashionMNIST (images flattened into vectors). The strong empirical performance of \our{} (best performance in 7 out of 9 cases, and second-best in the remaining 2, see~\Cref{tab:challenging_datasets}), indicates that the proposed mechanism remains effective in such challenging conditions.

\begin{table}[t]
\centering
\caption{Evaluation of \our{} and baseline methods on additional highly dimensional and sparse datasets.}
\footnotesize
\label{tab:challenging_datasets}
\setlength{\tabcolsep}{2pt}
\resizebox{\linewidth}{!}{
\begin{tabular}{l ccc|ccc|ccc}
\toprule
& \multicolumn{3}{c|}{\textbf{1-shot}} & \multicolumn{3}{c|}{\textbf{5-shot}} & \multicolumn{3}{c}{\textbf{10-shot}} \\
\midrule
\textbf{Method} &
{CNAE} &
{DEV} &
{Fashion} &
{CNAE} &
{DEV} &
{Fashion} &
{CNAE} &
{DEV} &
{Fashion} \\
\midrule
\midrule
CatBoost & 51.03 & 9.57 & 37.31 & 72.98 & 36.51 & 61.99 & 80.06 & 46.87 & 69.83 \\
kNN & 39.28 & 17.92 & 39.43 & 41.38 & 25.81 & 52.27  & 44.03 & 33.99 & 57.34 \\
TabPFN & 20.18 & - & 36.98 & 29.60 & - & 56.79 & 32.59 & - & 64.43 \\
SubTab & 42.59 & 23.95 & 25.14 & 64.81 & 31.01 & 54.91 & 71.14 & 40.01 & 63.02 \\
VIME & 23.53 & 18.95 & 45.96 & 52.13 & 26.80 & 57.04 & 62.00 & 35.30 & 60.66 \\
Scarf & 40.93 & 9.19 & 36.36 & 53.97 & 10.55 & 45.69 & 59.59 & 11.82 & 48.34  \\
T-JEPA & 41.57 & 11.51 & 27.34 & 63.76 & 26.05 & 42.64 & 72.51 & 37.49 & 50.32 \\
PseudoL & 47.04 & 9.83 & 25.15 & \textbf{76.03} & 28.26 & 41.03 & \textbf{83.55} & 41.38 & 62.17 \\
MeanT & 29.71 & 9.24 & 24.18 & 66.73 & 35.86 & 60.64 & 75.95 & 47.45 & 68.78 \\
SAINT & 45.27 & 16.82 & 36.61 & 65.12 & 22.94 & 49.24 & 67.96 & 43.45 & 62.15 \\
STUNT & 57.12 & 13.51 & 47.66 & 75.21 & 22.93 & 59.86 & 76.79 & 28.94 & 61.01 \\
D2R2-c & 26.91 & 3.99 & 32.19 & 35.91 & 4.11 & 35.02 & 44.36 & 28.94 & 41.15 \\
\midrule
\our{} & \textbf{64.76} & \textbf{27.83} & \textbf{52.74} & 75.31 & \textbf{47.98} & \textbf{67.96} & 80.19 & \textbf{58.49} & \textbf{72.13} \\
\bottomrule
\end{tabular}
}
\end{table}

\subsection{Analysis of \full{}} \label{sec:analysis}

In this section, we evaluate the validity of the proposed \full{} as an unsupervised pretraining objective.
For this purpose, we analyze its stability and the semantic relationship of the positive pairs created by \our{}. 
We present the example results for the OPT dataset in~\Cref{fig:neighbor_analysis}, and the results for the remaining datasets in~\Cref{sec:analysis_appendix}. We describe our analysis and findings in the following paragraphs.

\paragraph{Sampling the immediate nearest neighbor as the target yields the best chance of pairing same-class samples.} In~\Cref{fig:knn_fraction} and~\ref{fig:knn_analysis_appendix}, we analyze whether it would be beneficial to take into account more nearest neighbors in the~\our{} objective.  We generate 100 random separations into feature and target views with the separation ratio of 0.2 and measure the proportion of same-class samples within an increasing pool of nearest neighbors of the target views.
It is evident that sampling from such an increased pool would in practice decrease the chance of drawing positive pairs of samples who share the same labels and misalign the resulting representation with the downstream task.

\paragraph{Representations trained with \our{} increase the semantical meaningfulness of the nearest-neighbor graphs.} In \Cref{fig:10nn_consistency} and~\ref{fig:10nn_consistency_appendix}, we analyze how the local neighborhood structure is transformed from the raw input space to the representation space produced by \our{}. To this end, we group data points according to the number of pairs of the same-class in the pool of $k=10$ nearest neighbors in the input space. We encode input data with a \our{}-trained model, and compute the average number of neighbors of the same-class in the latent space. Despite being trained with the input-space nearest neighbor graph as the objective, the \our{}-trained representations recover an increased number of semantically related sample pairs,
even in the case of samples which do not lie closely to other samples from the same class in the data space, further evidencing the validity of \our{}.

\begin{figure*}[t]
    \centering
    \begin{subfigure}[t]{0.32\textwidth}
        \includegraphics[width=\linewidth]{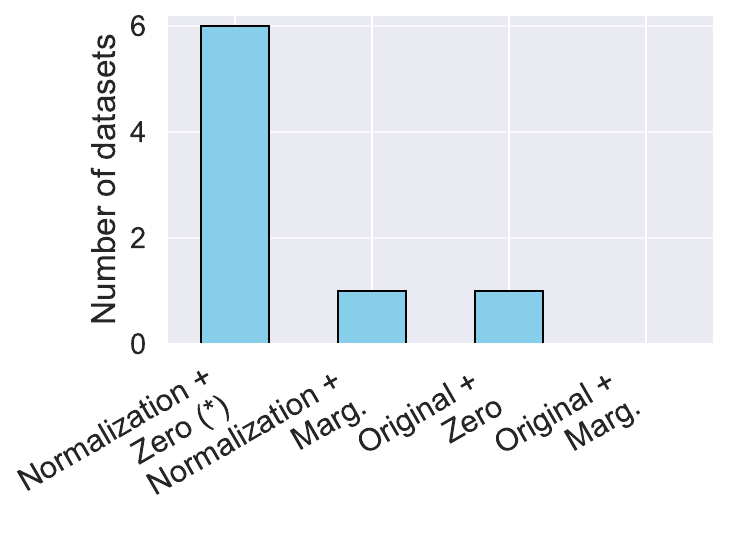}
        \caption{Ablation of data preprocessing techniques.}
        \label{fig:ablation_norm_imputation}
    \end{subfigure}
    \hfill
    \begin{subfigure}[t]{0.32\textwidth}
        \includegraphics[width=\linewidth]{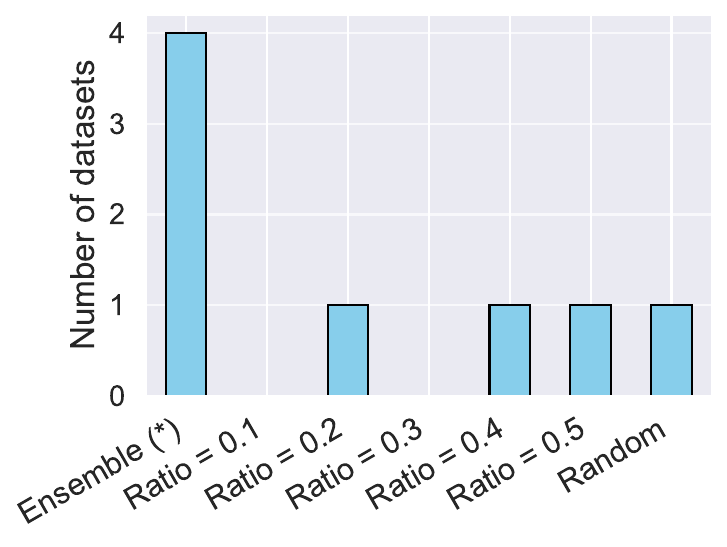}
        \caption{
            Ablation of separation ratio and model ensembling. 
        }
        \label{fig:ablation_masking_ensembling}
    \end{subfigure}
    \hfill
    \begin{subfigure}[t]{0.32\textwidth}
        \includegraphics[width=\linewidth]{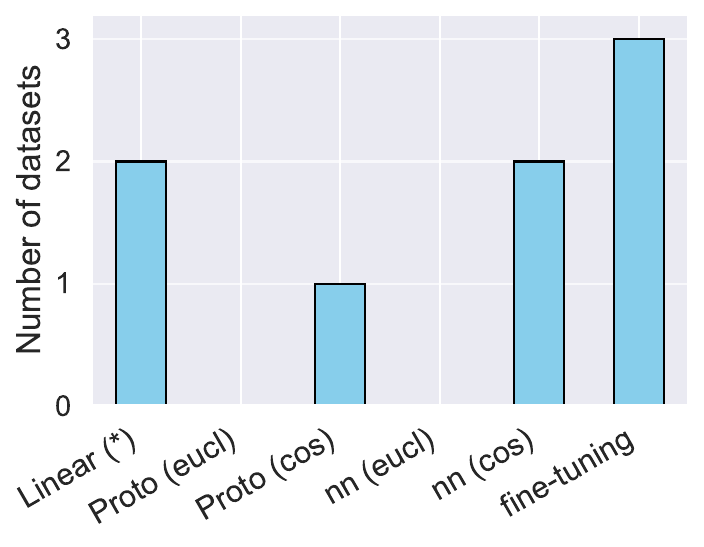}
        \caption{Ablation of few-shot classification approaches.}
        \label{fig:ablation_classifier}
    \end{subfigure}
    \caption{Ablation of the design aspects of \our{} (\textbf{(*)} denotes the default setting of \our{}). In the barplots, we report the number of datasets in which a given variant of \our{} performs best.}
    \label{fig:ablation}
\end{figure*}

\subsection{Ablation study} \label{sec:ablation}
In this section, we empirically demonstrate the usefulness of the core components of \our{}: the conditioned projector and the usage of the MLP encoder. We also ablate several design choices of \our{}: data preprocessing, separation ratio, and the choice of the multi-shot classifier. Unless otherwise specified, model variants are evaluated in terms of 5-shot classification accuracy with 5 random seeds.
We detail the variants of the model and report the results in~\Cref{tab:conditioned_projector,tab:mlp_and_ft_transformer} (core components), and in~\Cref{tab:ablation_norm_imputation,tab:ablation_masking_ensembling,tab:ablation_classifier_5shot} (design choices), which are also summarized in~\Cref{fig:ablation}.

\paragraph{Projector conditioning
(\Cref{tab:conditioned_projector}).} To understand the impact of conditioning the projector on separation information, we compare variants with and without conditioning (i.e. with a projector that processes only the encoder output) across 1-, 5-, and 10-shot learning downstream tasks. 
The variant with the conditioned projector generally performs better in all tasks. Notably, its advantage increases with the decrease in the amount of labeled data. 
This further suggests that conditioning helps to learn a representation of an improved quality, which is especially crucial under extremely low-data constraints.

\begin{table}[t]
\centering
\caption{Comparison of \our{} with conditioned and unconditioned projector.}
\label{tab:conditioned_projector}
\setlength{\tabcolsep}{1.5pt}
\footnotesize
\begin{tabular}{l c c c c c c c c c}
\toprule
\textbf{Condit.} &
\multicolumn{1}{c}{CMC} &
\multicolumn{1}{c}{DIA} &
\multicolumn{1}{c}{DNA} &
\multicolumn{1}{c}{INC} &
\multicolumn{1}{c}{KAR} &
\multicolumn{1}{c}{OPT} &
\multicolumn{1}{c}{PIX} &
\multicolumn{1}{c}{SEM} \\
\midrule
\multicolumn{9}{c}{\textit{1-shot}} \\
\midrule
Yes (*) & \textbf{36.76} & 
\textbf{61.14} & 
\textbf{66.79} & \textbf{62.89} & \textbf{76.40} & \textbf{78.94} & \textbf{83.06} & \textbf{61.11} \\ 
 No & 36.12 & 57.65 & 62.12 & 61.79 & 69.40 & 74.99 & 80.29 & 58.50 \\
 \midrule
 \multicolumn{9}{c}{\textit{5-shot}} \\
 \midrule
 Yes (*) & \textbf{42.85} & 
\textbf{69.54} & 
79.86 & \textbf{71.28} & 87.59 & \textbf{90.11} & 91.88 & \textbf{79.41} \\ 
 No & 40.96 & 65.59 & \textbf{81.61} & 71.06 & \textbf{88.05} & 90.02 & \textbf{92.60} & 78.62 \\
  \midrule
  \multicolumn{9}{c}{\textit{10-shot}} \\
  \midrule
 Yes (*) & \textbf{46.30} & 
\textbf{73.61} & 
\textbf{83.59} & 72.68 & \textbf{90.88} & 92.62 & 93.88 & 84.11 \\ 
 No & 43.10 & 67.59 & 80.39 & \textbf{73.69} & 88.99 & \textbf{92.86} & \textbf{94.21} & \textbf{84.28} \\
\bottomrule
\end{tabular}
\end{table}

\begin{table}[t]
\centering
\caption{Comparison of \our{} with MLP and FT-Transformer~\cite{gorishniy2023revisitingdeeplearningmodels} encoders in terms of 5-shot accuracy. The MLP architecture yields consistently stronger representations.}
\label{tab:mlp_and_ft_transformer}
\setlength{\tabcolsep}{1.5pt}
\footnotesize
\begin{tabular}{l c c c c c c c c c}
\toprule
\textbf{Encoder} &
\multicolumn{1}{c}{CMC} &
\multicolumn{1}{c}{DIA} &
\multicolumn{1}{c}{DNA} &
\multicolumn{1}{c}{INC} &
\multicolumn{1}{c}{KAR} &
\multicolumn{1}{c}{OPT} &
\multicolumn{1}{c}{PIX} &
\multicolumn{1}{c}{SEM} \\
\midrule
\midrule
FT-Trans. & 39.02 & 64.69& 73.17 & 69.51 & 84.14 & \textbf{91.47} & 86.63 & 78.59 \\ %
 MLP (*) & \textbf{42.85} & \best{69.54} & \textbf{79.86} & \textbf{71.28} & \best{87.59} & {90.11} & \best{91.88} & \best{79.41} \\ %
\bottomrule
\end{tabular}
\end{table}

\paragraph{Encoder architecture (\Cref{tab:mlp_and_ft_transformer}).} To understand whether \our{} would benefit from a more complex Transformer architecture, which is increasingly popular in the literature on tabular data~\cite{hollmanntabpfn,thimoniert}, we compare the standard 2-layer MLP encoder with a variant that uses an FT-Transformer architecture~\cite{gorishniy2023revisitingdeeplearningmodels} as an encoder. We find that the variant of \our{} with the MLP encoder yields better performance in almost all datasets.

\paragraph{Data preprocessing (\Cref{fig:ablation_norm_imputation} / \Cref{tab:ablation_norm_imputation}).} We ablate the usefulness of data normalization and two variants of missing data imputation: zero filling and sampling column values from marginal distribution. 
In most cases, the combination of data normalization and zero imputation yields representations of the highest quality.

\paragraph{Separation ratio (\Cref{fig:ablation_masking_ensembling} / \Cref{tab:ablation_masking_ensembling}).} We ablate the choice of target / feature separation ratio and compare it with an ensemble of encoders trained with all ratios, 
as well as randomly sampling the separation ratio during training. Although certain ratios and random sampling yield the best results for several datasets, we find that ensembles of encoders perform most reliably. 

\paragraph{Few-shot classifier (\Cref{fig:ablation_classifier} / \Cref{tab:ablation_classifier_5shot}).} We compare several choices of learning the classifier on top of the pretrained representation from the support data in the multi-shot setting. 
Linear probing is the simplest and robust approach. Although fine-tuning and k-NN classifiers are also competitive, they yield only slight improvements at the cost of increased complexity and variance in accuracy.

\section{Conclusion}

In this paper, we introduce \full{} (\our{}), a novel Semi-supervised Few-Shot Learning framework designed for tabular data.
\our{} uses the powerful Joint-Embedding Architecture (JEA) paradigm to pretrain its representations, while avoiding the problematic need to for manual data augmentation design -- the issue that has prevented the use of JEAs for tabular data in the past.
Instead, our core idea is to separate the data "at birth" into two independent, complementary subspaces and align the representations of one subspace to mirror the nearest-neighbor correspondence of the data in the second subspace. 
We demonstrate that this pretraining task indeed captures the semantic correspondence in a wide variety of tabular datasets.
\our{} achieves impressive performance in few-shot learning on various tabular datasets, confirming its effectiveness.
Our findings open new avenues for further investigations into tabular representation learning and are a useful foundation for data-constrained applications.

\section*{Acknowledgements}
The work of K. Jurek and M. \'Smieja was supported by the National Science Centre (Poland), grant no. 2023/50/E/ST6/00169. 
The research of M. Przewięźlikowski was supported by the National Science Centre (Poland), grant no. 2023/49/N/ST6/03268.
The work of W. Batko was funded by the program Excellence Initiative – Research University at the Jagiellonian University in Kraków.
We gratefully acknowledge Polish high-performance computing infrastructure PLGrid (HPC Center: ACK Cyfronet AGH) for providing computer facilities and support within computational grant no. PLG/2025/018312 and PLG/2025/018945

\bibliographystyle{siamplain}
\bibliography{ref}

\appendix

\section{Experiment details}

\subsection{Datasets}
\label{sec:datasets}

The details of the datasets are presented in \Cref{sec:data}.

\begin{table*}[h]
\centering
\caption{Overview of the datasets used in the experiments, including the number of instances, proportion of numerical and categorical features, and the number of classes.} \label{sec:data}
\begin{tabular}{clccc}
\toprule
\textbf{Dataset code} &
\textbf{Dataset}   & \# \textbf{Instances} & \# \textbf{Features (num., cate.)} & \# \textbf{Classes} \\
\midrule
\midrule
CMC & cmc       & 1473   & 9 (2,7)          & 3 \\
DIA & diabetes  & 768    & 8 (8,0)          & 2  \\
DNA & dna       & 3186   & 180 (0,180)      & 3  \\
INC & income    & 48842  & 14 (6,8)         & 2  \\
KAR & karhunen  & 2000   & 64 (64,0)        & 10 \\
OPT & optdigits & 5620   & 64 (64,0)        & 10 \\
PIX & pixel     & 2000   & 240 (0,240)      & 2  \\
SEM & semeion   & 1593   & 256 (256,0)      & 10 \\
CNAE & cnae     & 1080   & 856 (856,0)      & 9 \\
DEV & devnagari & 92000  & 1024 (1024,0) & 46 \\
Fashion & fashion-mnist  & 70000 & 784 (784,0) & 10 \\
\bottomrule
\end{tabular}
\end{table*}

\subsection{Baselines} \label{app:methods}

Together with \our{}, we report the performance of twelve methods taken from \cite{liu2024d2r2} with publicly available implementations that are adapted for tabular data. These methods represent three types of baselines:
\begin{enumerate}
    \item \textbf{Supervised.} CatBoost~\cite{prokhorenkova2018catboost} is considered a shallow SOTA approach to tabular data; k-NN~\cite{peterson2009k} works well for the few-shot case; TabPFN~\cite{hollmanntabpfn} represents a transformer-based zero-shot technique, which can be applied to small datasets.  
    \item \textbf{Self-supervised.} VIME~\cite{yoon2020vime}, SubTab~\cite{ucar2021subtab} and SCARF~\cite{bahri2021scarf}, T-JEPA~\cite{thimoniert}, and SAINT~\cite{somepalli2021saint} represent the typical SSL approaches for tabular data. The representations acquired from those models are used to conduct Center Prototype Classification or linear probing. 
    \item \textbf{Semi-supervised learning}. Mean Teacher~\cite{tarvainen2017mean} is a method which uses the consistency loss between the teacher output and student output. Pseudo-label~\cite{lee2013pseudo} predicts the labels of unlabeled data and then uses them as learning signal.
\end{enumerate}

STUNT~\cite{nam2023stunt} and D2R2~\cite{liu2024d2r2} were evaluated using authors' repositories with hyperparameters selection procedures implemented there. 
For D2R2, we run the variant denoted by D2R2-c,  which uses mean support embeddings as the classifier. The default D2R2 uses an instance-wise iterative prototype scheme, additionally using query data for class
prototype estimation. This is not consistent with the inductive setting, where queries are unseen during classifier training.

\subsection{Hyperparameters}

We report the values of the hyperparameters used by \our{} in \Cref{tab:hyper}.

\begin{table}[h]
\centering
\caption{\our{} hyperparameters} \label{tab:hyper}
\begin{tabular}{c c}
\toprule
\textbf{Hyperparameter} & \textbf{Value}  \\ %
\midrule
\midrule
\multicolumn{2}{c}{Pretraining}\\
\midrule
Epochs & 10.000   \\ %
Learning rate & 0.001 \\ %
Optimizer & Adam~\cite{kingma2014adam} \\
Batch size & 1024  \\ %
Early stopping patience & 100  \\ %
Encoder depth & 2 \\
Encoder hidden size & 1024  \\ %
Encoder output size & 256  \\ %
Projector depth & 2 \\
Projector hidden size & 1024  \\ %
Projector output size & 256  \\ %

\midrule
\multicolumn{2}{c}{Few-shot classification}\\

\midrule
Epochs & 10.000   \\ %
Learning rate & 0.001 \\ %
Optimizer & Adam~\cite{kingma2014adam} \\
\bottomrule
\end{tabular}
\end{table}

\subsection{Implementation details}

We implement \our{} in PyTorch~\cite{paszke2019pytorch}. We include the codebase as supplementary material and will publish it along with the paper.
All of the experiments described in the paper were run on a single NVidia-V100 GPU.

\section{Theoretical analysis} \label{app:theory}

\subsection{Summary of the results}

We present a theoretical analysis, which confirms that, on average, nearest neighbors in target views give samples of the same class.

For transparency, let us consider a simplified model of data on $\R^D$ representing two Gaussian classes:
\[
X\mid Y=c \sim \mathcal N(\nu_c,I),\qquad X\mid Y=c' \sim \mathcal N(\nu_{c'},I).
\]
Working with arbitrary covariances (but identical for both classes) is possible but requires data whitening, which further complicates the derivation. Therefore, we restrict our attention to the isotropic case. Let us denote the class-mean difference $\Delta = (\nu_{c'}-\nu_c)$. We can calculate the squared separation between Gaussian means as:
\[
\delta^2 \;=\; \Delta^T \Delta = \sum_{i=1}^D \Delta_i^2.
\]

We are interested in the probability of the \emph{mismatch} -- the event that the nearest neighbor of a sample $x$ has a different class than $x$. Observe that the probability of mismatch decreases as the separation becomes larger.  

Thus, we are interested in preserving separation in target views, where the nearest neighbors are calculated. If we consider a $n$-dimensional target view defined by a subset $S \subset \{1,\ldots,D\}$ with $|S|=n$, then the squared separation in this view equals: 
$$
\delta_S^2 = \sum_{i \in S} \Delta_i^2.
$$
As can be seen, we cannot guaranty that every target view provides a well-separated class-mean difference. However, we can show that with small probability we select a target view, where the probability of mismatch is large.

Our analysis consists of two parts:
\begin{enumerate}
    \item We bound the probability of mismatch in the target view $S$ from the above by the exponential term dependent on the separation $\delta_S^2$ in this view.
    \item Taking the expectation, we show that on average, the probability of mismatch in the target view decays exponentially in the separation $\delta^2$ in the original data space $\R^D$.
\end{enumerate}
The last fact is crucial because it guaranties that the separation is preserved on average if only initial classes are also separated enough. In consequence, we are allowed to perform nearest neighbor search in target view to produce meaningful positive pairs without the risk that the nearest neighbor will have a different class.

\subsection{Bound on the mismatch probability}

\begin{lemma} \label{lem:isotropic-mismatch}
Let \(X\), \(Y\) be independent draws from \(\mathcal{N}(\mu,I_d)\) (the same class)
, and let \(Z\) be an independent draw from \(\mathcal{N}(\mu+\Delta,I_d)\)
(a different class), where \(\Delta\in\mathbb{R}^d\) denotes the mean offset.
Define the mismatch event
\[
\mathcal{E} \;=\; \{\,\|Z-X\|^2 \le \|Y-X\|^2\,\},
\]
i.e., the sample of the other-class \(Z\) is no further than the sample of the same-class \(Y\).
There exists a global constant $C > 0$ such that
\[
\Pr(\mathcal{E})
\le 2 \exp(-C \|\Delta\|^2)
\]
\end{lemma}

\begin{proof}
Denote
\[
X=\mu+\varepsilon_X,\qquad Y=\mu+\varepsilon_Y,\qquad Z=\mu+\Delta+\varepsilon_Z,
\]
with \(\varepsilon_X,\varepsilon_Y,\varepsilon_Z\overset{\text{i.i.d.}}{\sim}\mathcal{N}(0,I_d)\).
Set
\[
U:=\varepsilon_Z-\varepsilon_X,\qquad V:=\varepsilon_Y-\varepsilon_X.
\]
Then \(U\) and \(V\) are independent vectors of \(\mathcal{N}(0,2I_d)\) and 
\[
\begin{aligned}
D &:= \|Z-X\|^2 - \|Y-X\|^2
   \;=\; \|\Delta+U\|^2 - \|V\|^2 \\
  &= \|\Delta\|^2 + 2\Delta^\top U + \big(\|U\|^2 - \|V\|^2\big).
\end{aligned}
\]
The mismatch event \(\mathcal{E}\) equals \(\{D\le 0\}\), that is,
\begin{equation}
2\Delta^\top U + \big(\|U\|^2-\|V\|^2\big) \le -\|\Delta\|^2. \tag{1}
\end{equation}

If the sum of two independent random terms is \(\le -\|\Delta\|^2\), then at least one
of them must be \(\le -\tfrac12\|\Delta\|^2\). Hence, by the union bound,
\begin{equation}
\begin{aligned}
    \Pr(D\le 0) \le{} & \Pr\!\Big(2\Delta^\top U \le -\tfrac12\|\Delta\|^2\Big) \\
                     & + \Pr\!\Big(\|U\|^2-\|V\|^2 \le -\tfrac12\|\Delta\|^2\Big).
\end{aligned}
\tag{2}
\end{equation}

We bound each term on the right-hand side.

\smallskip\noindent\textbf{(i) Linear term.}
The scalar random variable \(G:=2\Delta^\top U\) is Gaussian with mean \(0\) and variance
\[
\mathrm{Var}(G) \;=\; 4 \mathrm{Var}(\Delta^\top U)
= 4\cdot \Delta^\top (2I_d)\Delta
= 8\|\Delta\|^2.
\]
Therefore, for \(a=\tfrac12\|\Delta\|^2\), the Gaussian tail bound gives
\begin{equation}
\begin{aligned}
\Pr\!\Big(2\Delta^\top U \le -a\Big)
& \le \exp\!\Big(-\frac{a^2}{2\mathrm{Var}(G)}\Big) \\
& = \exp\!\Big(-\frac{(\tfrac12\|\Delta\|^2)^2}{16\|\Delta\|^2}\Big) \\
& = \exp\!\Big(-\frac{\|\Delta\|^2}{64}\Big).
\end{aligned}
\tag{3}
\end{equation}

\smallskip\noindent\textbf{(ii) Quadratic-difference term.}
Note that \(\|U\|^2\) and \(\|V\|^2\) are independent and each has the distribution
of \(2\) times a \(\chi^2_d\) random variable, hence \(\mathbb{E}\|U\|^2=\mathbb{E}\|V\|^2=2d\).
We bound the probability that their difference is \(\le -\tfrac12\|\Delta\|^2\) by
controlling the deviations of each chi-square from its mean. For any \(t>0\),
\[
\begin{aligned}
\Pr \Big( \|U\|^2 - \|V\|^2 \le -t \Big) 
&\le \Pr \Big( \|U\|^2 \le 2d - \tfrac{t}{2} \Big) \\
&\quad + \Pr \Big( \|V\|^2 \ge 2d + \tfrac{t}{2} \Big).
\end{aligned}
\]
Standard chi-square concentration (e.g. \(\Pr(\chi^2_d \le d(1-\epsilon))\le \exp(-d\epsilon^2/4)\)
for \(0<\epsilon<1\), and \(\Pr(\chi^2_d \ge d(1+\epsilon))\le \exp(-d\epsilon^2/4)\) for \(\epsilon>0\))
implies that, for \(t=\tfrac12\|\Delta\|^2\), both probabilities are bounded by \(\exp(- \frac{\|\Delta\|^4}{256d})\) and, in consequence,
\begin{equation}
\Pr\!\Big(\|U\|^2-\|V\|^2 \le -\tfrac12\|\Delta\|^2\Big)
\le 2\exp\!\Big(- \frac{\|\Delta\|^4}{256d})\Big).
\tag{4}
\end{equation}

\smallskip Combining (2),(3),(4) yields
\[
\Pr(D\le 0)
\le \exp\!\Big(-\frac{\|\Delta\|^2}{64}\Big)
+2\exp\!\Big(- \frac{\|\Delta\|^4}{256d})\Big).
\]
Finally, we may consider two cases: (i) $\|\Delta\|^2 \geq d$, and (ii) $\|\Delta\|^2 < d$. In both cases, we can find a common factor $C > 0$ such that
\[
\Pr(D \leq 0)
\le 2 \exp(-C \|\Delta\|^2)
\]
which completes the proof.
\end{proof}

\subsection{Expected mismatch probability}

Let $S$ define $n$ dimensional target view. By $P_S(X)$ we denote a coordinate-wise projection of a sample $X \subset \R^D$ onto a target view $S$. Given two samples $X,Y \subset \R^D$ from a class $c$, a other sample $Z \subset \R^D$ from class $c' \neq c$, we define the probability of mismatch in a view $S$ as
\[
\begin{aligned}
p_{\mathrm{mismatch}}(S) 
& := \Pr \big( \|P_S(Z)-P_S(X)\|^2 \\
& \qquad \le \|P_S(Y)-P_S(X)\|^2 \big).
\end{aligned}
\]

The following theorem shows that the expected mismatch probability in the target view decays exponentially in the preserved fraction $n/D$ of the full separation $\delta^2$ and it does not depend on the separation in the individual target views.
\begin{theorem}
Let us consider data on $\R^D$ representing two Gaussian classes:
\[
X\mid Y=c \sim \mathcal N(\nu_c,I),\qquad X\mid Y=c' \sim \mathcal N(\nu_{c'},I) 
\]
with the squared separation term $
\delta^2 \;=\; (\nu_{c'}-\nu_c)^T (\nu_{c'}-\nu_c)$. Let $S$ be chosen uniformly at random among all $\binom{D}{n}$ $n$ dimensional target views. Then the expected probability of mismatch can be bounded as
\[
\mathbb E_S[p_{\mathrm{mismatch}}(S)] \le 2\exp\!\Big(-C \frac{n \delta^2}{D}\Big),
\]
for some constant $C > 0$.
\end{theorem}

\begin{proof}
The projected squared separation equals
$\delta_S^2=\sum_{i\in S} a_i$ and Lemma \ref{lem:isotropic-mismatch} implies that
\[
p_{\mathrm{mismatch}}(S) \le 2\exp\!\Big(-C \delta_S^2\Big),
\]
for some constant $C>0$.

Taking the expectation over $S$ and using Jensen's inequality ($e^{-x}$ is convex downward, so $\mathbb E[e^{-X}]\le e^{-\mathbb E[X]}$), we get:
\[
\mathbb E_S[p_{\mathrm{mismatch}}(S)] \le 2\exp\!\Big(-C\mathbb E_S[\delta_S^2]\Big).
\]
Finally, $\mathbb E_S[\delta_S^2] = \frac{n}{D}\delta^2$ by the symmetry of uniform sampling without replacement, which gives:
\[
\mathbb E_S[p_{\mathrm{mismatch}}(S)] \le 2\exp\!\Big(-C \frac{n \delta^2}{D}\Big).
\]
\end{proof}

\section{Additional experimental results}

\begin{figure}[h]
    \centering
    \begin{subfigure}{0.49\textwidth}
        \centering
        \includegraphics[width=\linewidth]{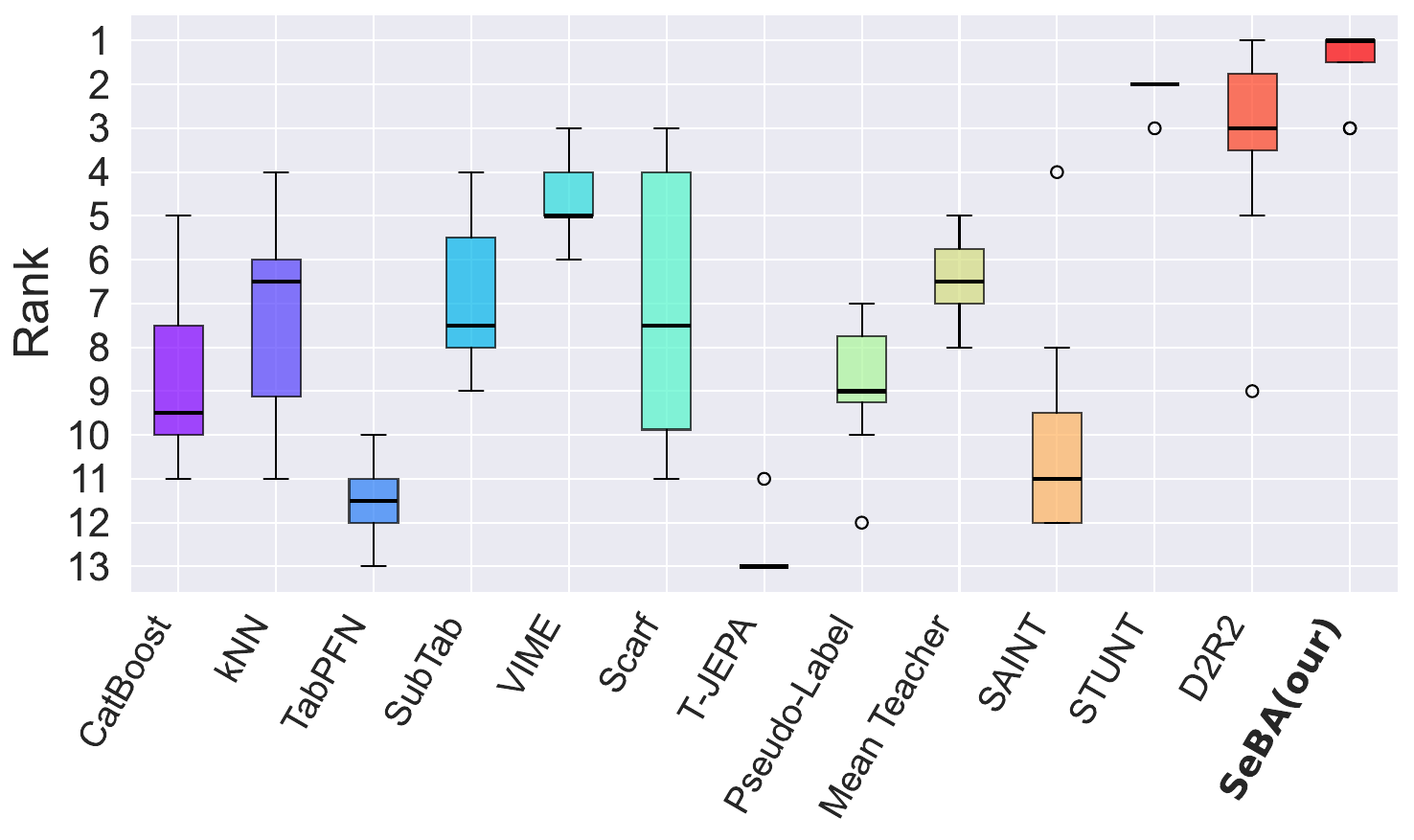}
        \caption{1-shot classification accuracy}
    \end{subfigure}
    \hfill
    \begin{subfigure}{0.49\textwidth}
        \centering
        \includegraphics[width=\linewidth]{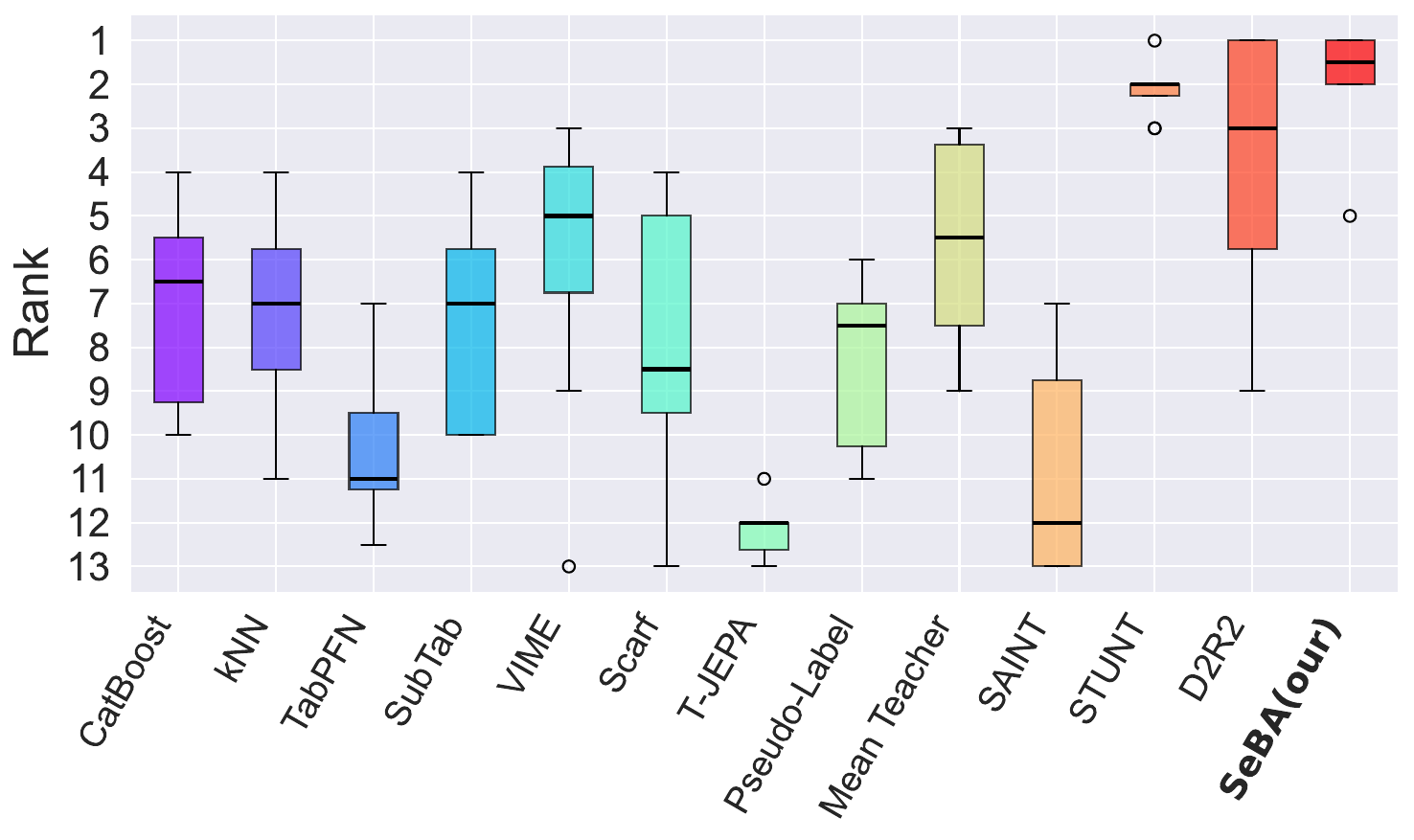}
        \caption{5-shot classification accuracy}
    \end{subfigure}
    \hfill
    \begin{subfigure}{0.49\textwidth}
        \centering
        \includegraphics[width=\linewidth]{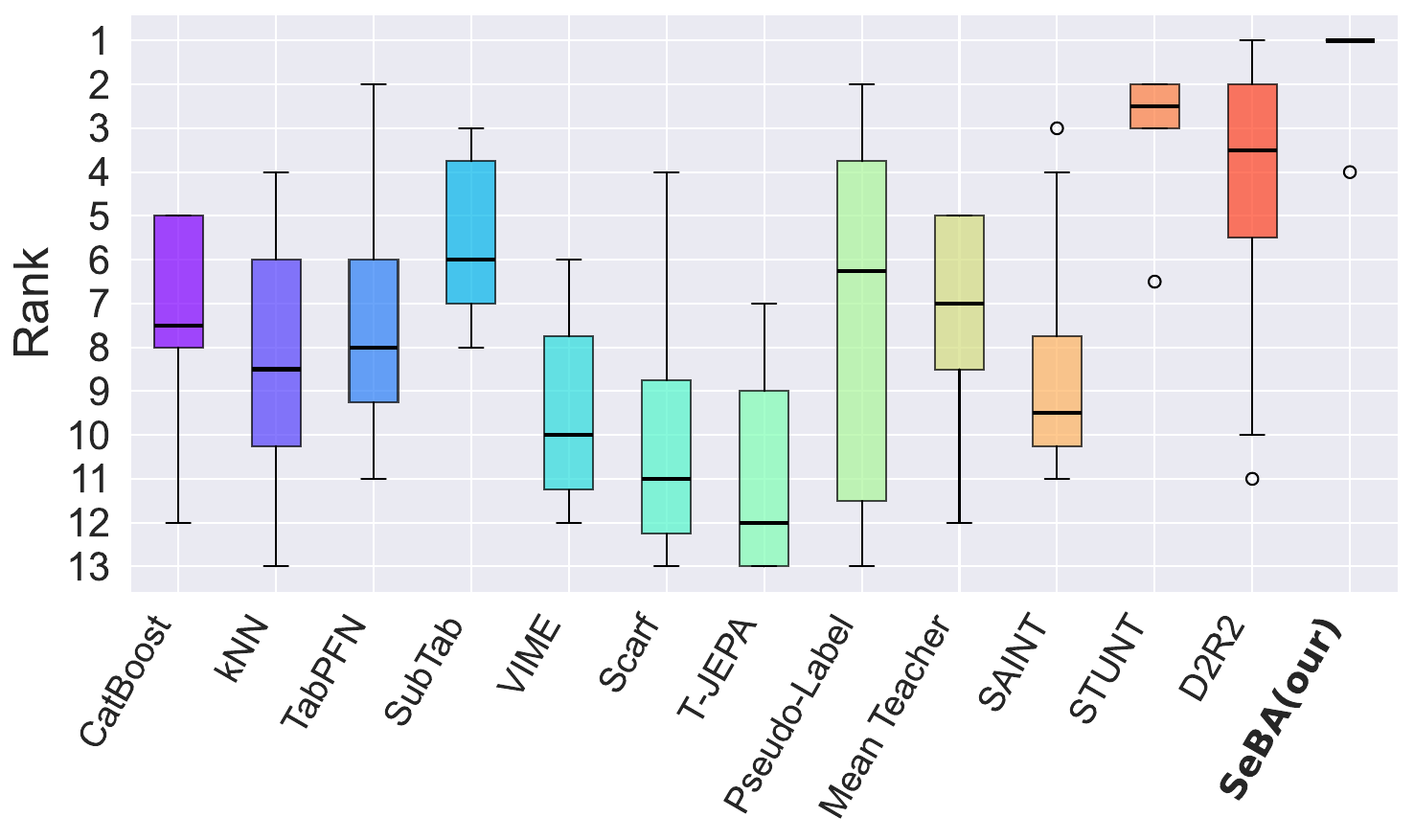}
        \caption{10-shot classification accuracy}
    \end{subfigure}
    \caption{Box-plots of few-shot classification ranks of benchmarked approaches. \our{} is the most consistently high-ranking method.
    }
    \label{fig:ranks}
    \end{figure}

\subsection{Detailed ablation study results}

In~\Cref{tab:ablation_norm_imputation,tab:ablation_masking_ensembling,tab:ablation_classifier_5shot}, we include the details of variants and results of ablations described in~\Cref{sec:ablation}.

\begin{table}[h!]
\centering
\caption{Ablation of normalizing the  numerical columns in tabular data (\textbf{Norm.}) and of the type of data imputation in separated columns (\textbf{Imput.}), where we compare zero-imputation (Zero), and sampling from the column's marginal distribution (Marg.). 
}
\label{tab:ablation_norm_imputation}
\setlength{\tabcolsep}{1.5pt}
\footnotesize
\resizebox{\linewidth}{!}{
\begin{tabular}{l c c c c c c c c c}
\toprule
\textbf{Norm.} & \textbf{Imput.} &
\multicolumn{1}{c}{CMC} &
\multicolumn{1}{c}{DIA} &
\multicolumn{1}{c}{DNA} &
\multicolumn{1}{c}{INC} &
\multicolumn{1}{c}{KAR} &
\multicolumn{1}{c}{OPT} &
\multicolumn{1}{c}{PIX} &
\multicolumn{1}{c}{SEM} \\
\midrule
\midrule
\multirow{2}{*}{True}   & Zero \textbf{(*)} & \textbf{42.85} & {69.54} & \textbf{79.86} & \textbf{71.28} & 87.59 & \best{90.11} & \best{91.88} & \best{79.41} \\ %

  & Marg. & 41.71 & \best{69.78} & 68.69 & 67.89 & 86.36 & 89.18 & 89.67 & 77.32 \\ %
\multirow{2}{*}{False} & Zero & 40.50 & 53.49 & 70.82 & 47.01 & \best{91.00} & 87.91 & 90.54 & 77.28 \\ %
                      & Marg. & 37.95 & 54.97 & 69.71 & 46.44 & 89.09 & 89.41 & 89.29 & 77.85 \\ %
\bottomrule
\end{tabular}
}
\end{table}
\begin{table}[h!]
\centering
\caption{Ablation of the separation ratios between the target and feature data views, compared with the ensemble of encoders trained with different ratios and randomly sampled mask ratios.
While encoders trained with constant separation ratios achieve strong performance, ensembling their predictions is the most robust in practice.
}
\label{tab:ablation_masking_ensembling}
\setlength{\tabcolsep}{1.5pt}

\resizebox{\linewidth}{!}{

\begin{tabular}{l c c c c c c c c}
\toprule
\textbf{Mode} & 
\multicolumn{1}{c}{CMC} & 
\multicolumn{1}{c}{DIA} & 
\multicolumn{1}{c}{DNA} & 
\multicolumn{1}{c}{INC} & 
\multicolumn{1}{c}{KAR} & 
\multicolumn{1}{c}{OPT} & 
\multicolumn{1}{c}{PIX} & 
\multicolumn{1}{c}{SEM} \\ 
\midrule
\midrule
Ensemble \textbf{(*)}& {42.85} & {69.54} & \textbf{79.86} &\textbf{71.28} & {87.59} & {90.11} & \textbf{{91.88}} & \best{79.41} \\ %
Ratio = 0.1 & 41.18 & 68.20 & 69.07 & 69.73 & 84.21 & 85.84 & 89.37 & 72.02 \\ %
Ratio = 0.2 & 41.56 & \textbf{69.73} & 73.64 & 68.14 & 86.36 & 88.83 & 91.13 & 76.04 \\ %
Ratio =  0.3 & {42.71} & 68.19 &  {77.01} & 68.40 & 87.53 & 90.29 & 91.24 & 77.33 \\ %
Ratio = 0.4 & 41.86 & 68.41 & 73.33 & 70.11 & 87.43 & \best{90.69} & 91.79 & 78.69 \\ %
Ratio = 0.5 & 41.16 & 68.26 & 72.69 & 70.74 & \best{88.11} & 90.51 & 89.83 & {78.94} \\ %
Random & \textbf{43.57} & 60.83 & 67.19 & 65.26 & 86.37 & 89.91 & 88.91 & 74.98 \\ %

\bottomrule
\end{tabular}
}
\end{table}
\begin{table}[h!]
\centering
\caption{Ablation of different ways of forming the many-shot classifier. We compare linear probing (Linear), using support data to form prototypes and assigning queries based on Euclidean od cosine distance (Proto eucl/cos), matching individual support representations as nearest neighbors based on Euclidean od cosine distance (nn eucl /cos), and fine-tuning the whole encoder along with the classifier (fine-tuning). 
}
\label{tab:ablation_classifier_5shot}
\setlength{\tabcolsep}{1.5pt}
\resizebox{\linewidth}{!}{

\begin{tabular}{l c c c c c c c c c}
\toprule
\textbf{Mode} & 
\multicolumn{1}{c}{CMC} & 
\multicolumn{1}{c}{DIA} & 
\multicolumn{1}{c}{DNA} & 
\multicolumn{1}{c}{INC} & 
\multicolumn{1}{c}{KAR} & 
\multicolumn{1}{c}{OPT} & 
\multicolumn{1}{c}{PIX} & 
\multicolumn{1}{c}{SEM} \\
\midrule
\midrule
Linear \textbf{(*)} & \textbf{42.85} & {69.54} & \textbf{79.86} & 71.28 & {87.59} & {90.11} & {91.88} & {79.41} \\ %
Proto (eucl) & 39.78 & 67.27 & 75.85 & 70.45 & 86.06 &  86.76 &  92.33 & 75.16 \\ %
Proto (cos) & 39.35 & 67.66 & {78.49} &  70.30 & 87.36 &  89.61 &  \best{92.71} &  75.50 \\ %
nn (eucl) & 40.31 & 66.57 & 72.68 &  69.22 & 86.26 &  89.74 &  92.37 & 76.41 \\ %
nn (cos) & 40.83 & 67.34 & 74.53 &  70.14 & \best{87.74} &  \best{91.09} &  92.55 &  77.74 \\ %
fine-tuning &  {42.52} & \textbf{70.31} & 76.30 & \best{71.63} & 85.27 & 89.09 &  91.71 &  \best{79.48} \\ %
\bottomrule
\end{tabular}
}
\end{table}

\subsection{Detailed analysis of \our{} results}
\label{sec:analysis_appendix}

We present detailed results for all datasets of the analysis performed in \Cref{sec:analysis}. 
In~\Cref{fig:knn_analysis_appendix} we measure the average fraction of neighbors sharing the same class label as a given data sample, when inspecting exactly $k$ nearest neighbors. The immediate nearest neighbor is the one with the highest probability of sampling a same-class data point in the majority of datasets, and in the remaining ones, this probability does not vary significantly when increasing the nearest-neighbor pool, which means that the immediate nearest neighbor remains a viable candidate.
In~\Cref{fig:10nn_consistency_appendix}, we compare the  label consistency between the baseline and the \our{} representation space.
Notably, \our{} representation increases the semantic consistency of nearest neighbors especially for data samples whose semantic consistency is low in the data space, while maintaining high consistency in samples who are already highly semantically consistent with their neighbors.

\begin{figure*}[h]
    \centering
    \begin{subfigure}{0.3\textwidth}
        \centering
        \includegraphics[width=\linewidth]{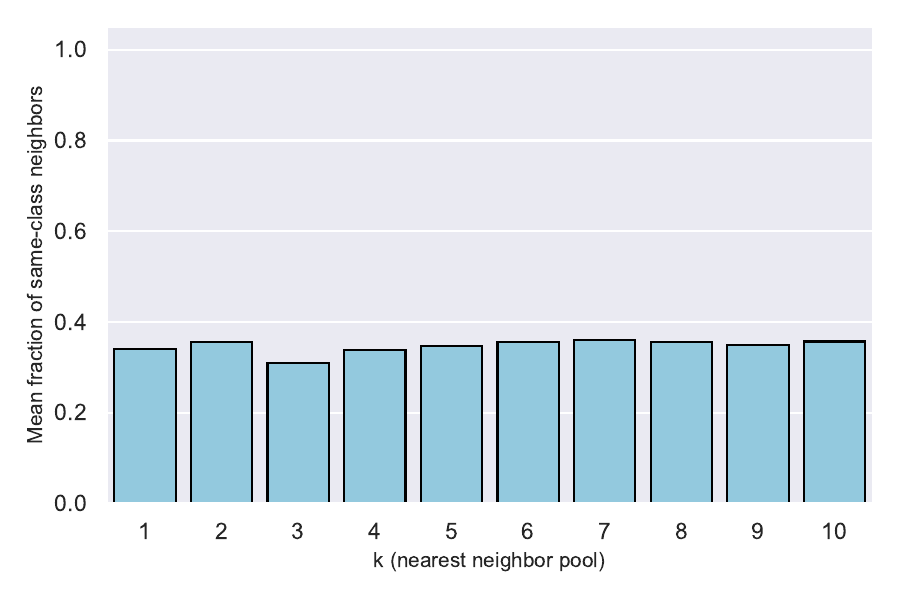}
        \caption{CMC}
    \end{subfigure}
    \hfill
    \begin{subfigure}{0.3\textwidth}
        \centering
        \includegraphics[width=\linewidth]{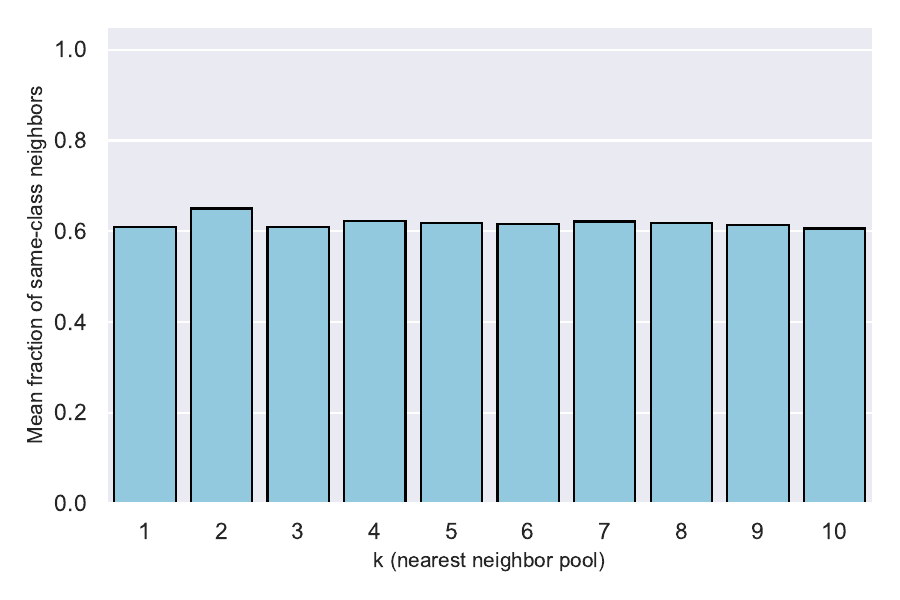}
        \caption{DIA}
    \end{subfigure}
    \hfill
    \begin{subfigure}{0.3\textwidth}
        \centering
        \includegraphics[width=\linewidth]{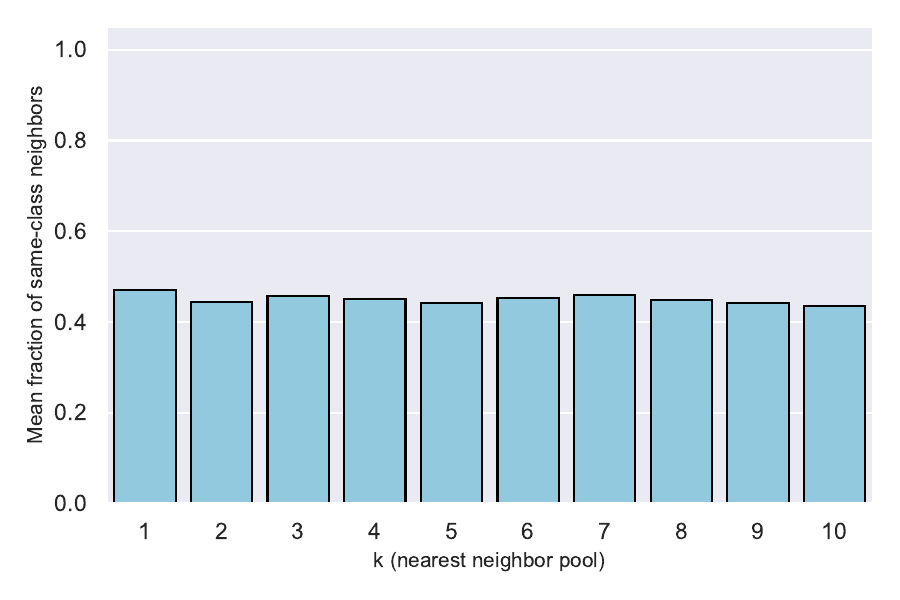}
        \caption{DNA}
    \end{subfigure}

    \vspace{0.5em}

    \begin{subfigure}{0.3\textwidth}
        \centering
        \includegraphics[width=\linewidth]{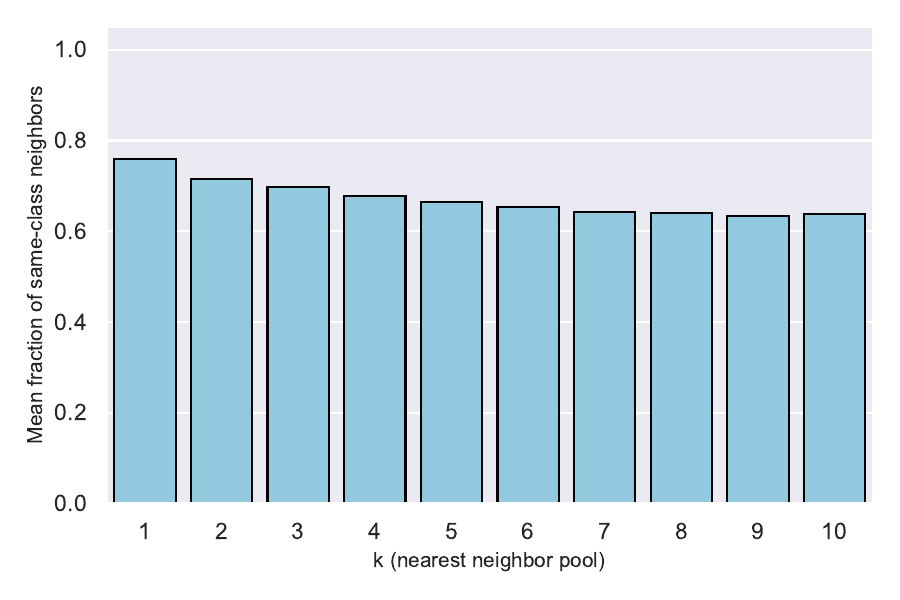}
        \caption{INC}
    \end{subfigure}
    \hfill
    \begin{subfigure}{0.3\textwidth}
        \centering
        \includegraphics[width=\linewidth]{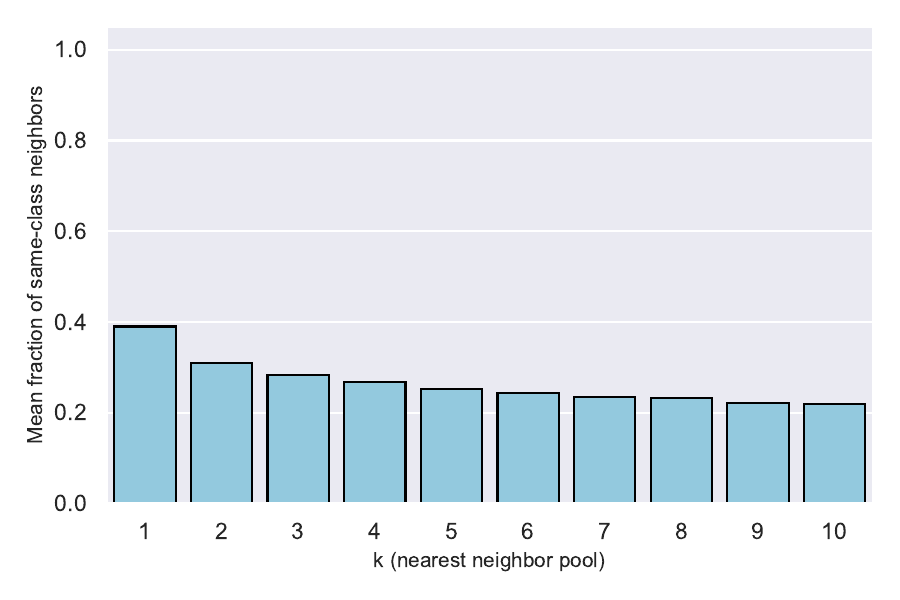}
        \caption{KAR}
    \end{subfigure}
    \hfill
    \begin{subfigure}{0.3\textwidth}
        \centering
        \includegraphics[width=\linewidth]{new_fig/exp9/optdigits_mean_fraction_vs_k.pdf}
        \caption{OPT}
    \end{subfigure}

    \vspace{0.5em}

    \begin{subfigure}{0.3\textwidth}
        \centering
        \includegraphics[width=\linewidth]{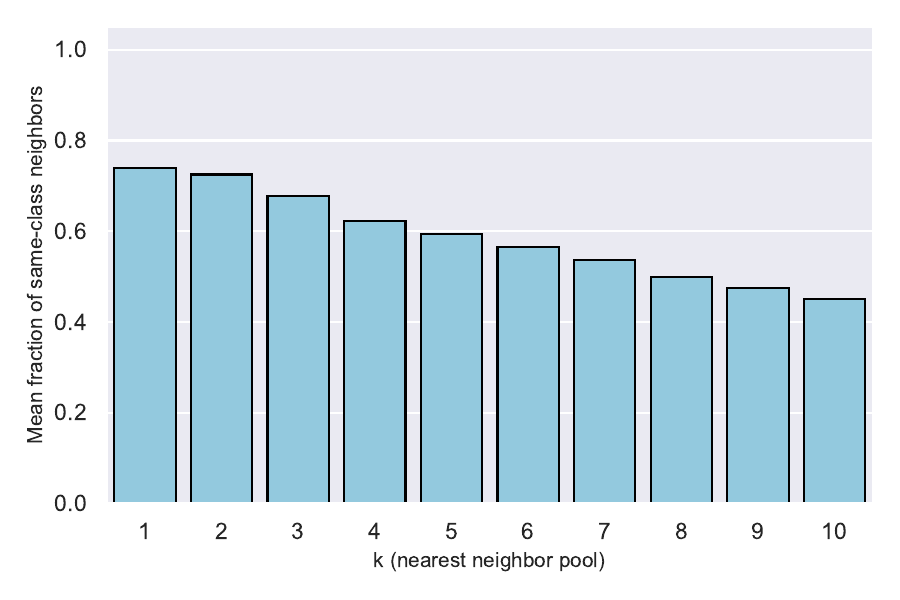}
        \caption{PIX}
    \end{subfigure}
    \hfil
    \begin{subfigure}{0.3\textwidth}
        \centering
        \includegraphics[width=\linewidth]{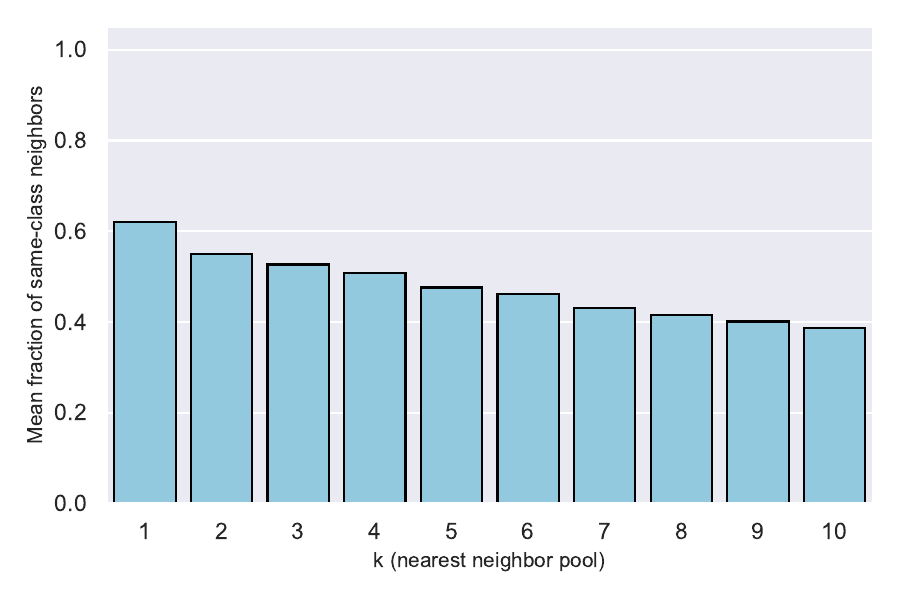}
        \caption{SEM}
    \end{subfigure}

    \caption{Average fraction of nearest neighbors sharing the same class label as a given data sample, when inspecting $k$ nearest neighbors.}
        \label{fig:knn_analysis_appendix}
\end{figure*}

\begin{figure*}[h]
    \centering
    \begin{subfigure}{0.3\textwidth}
        \centering
        \includegraphics[width=\linewidth]{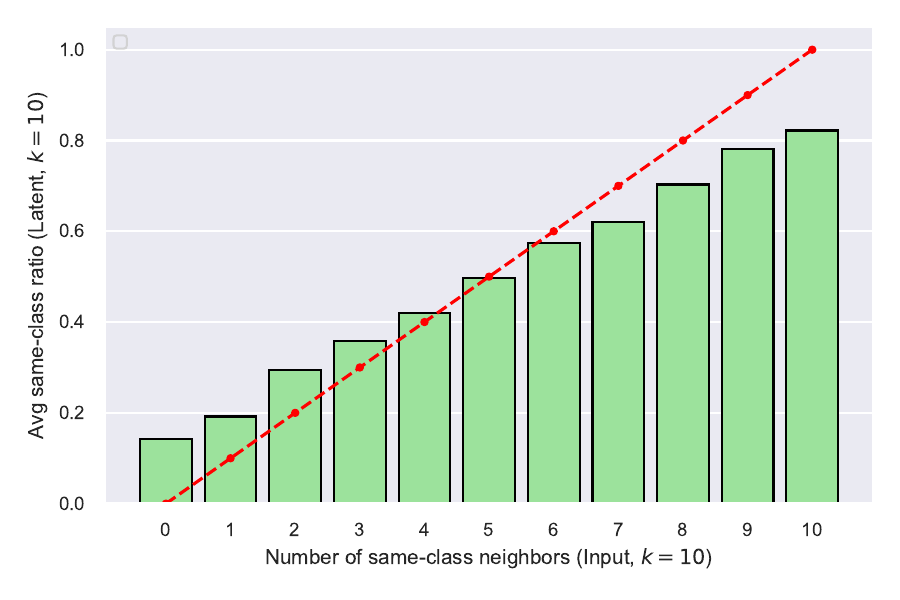}
        \caption{CMC}
    \end{subfigure}
    \hfill
    \begin{subfigure}{0.3\textwidth}
        \centering
        \includegraphics[width=\linewidth]{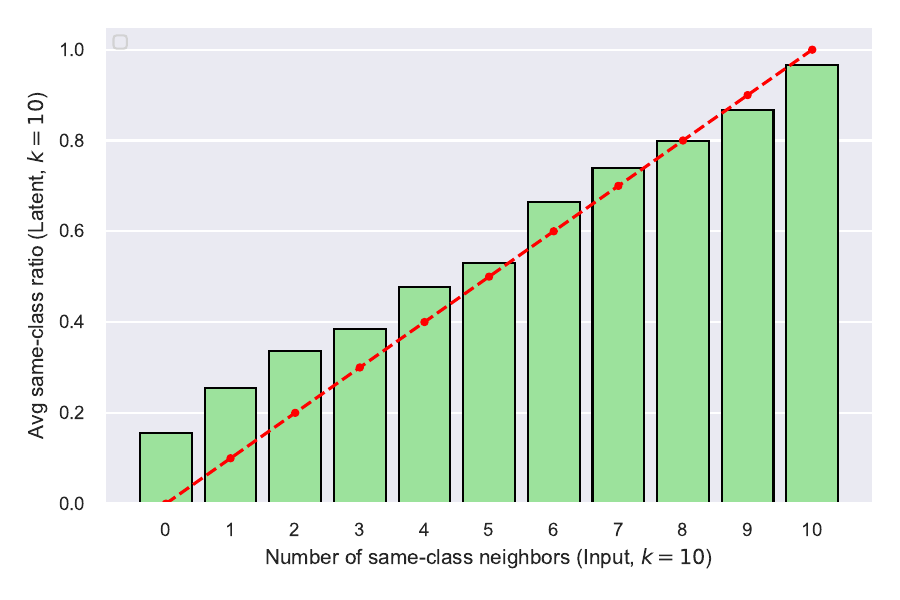}
        \caption{DIA}
    \end{subfigure}
    \hfill
    \begin{subfigure}{0.3\textwidth}
        \centering
        \includegraphics[width=\linewidth]{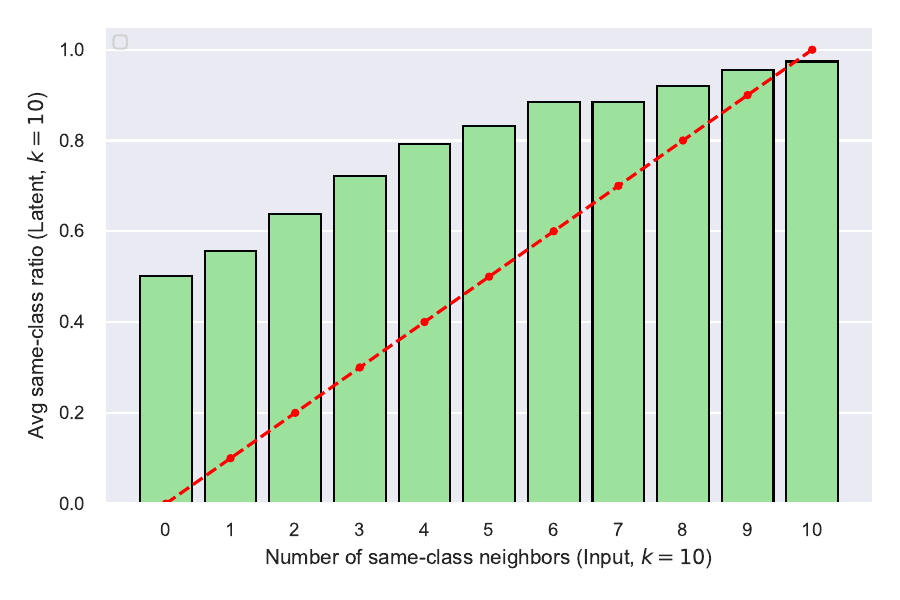}
        \caption{DNA}
    \end{subfigure}

    \vspace{0.5em}

    \begin{subfigure}{0.3\textwidth}
        \centering
        \includegraphics[width=\linewidth]{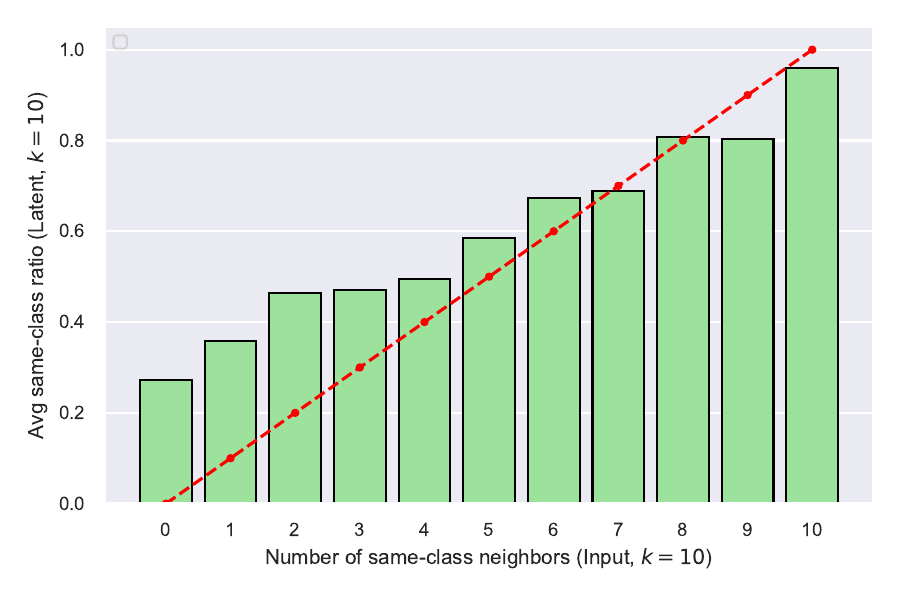}
        \caption{INC}
    \end{subfigure}
    \hfill
    \begin{subfigure}{0.3\textwidth}
        \centering
        \includegraphics[width=\linewidth]{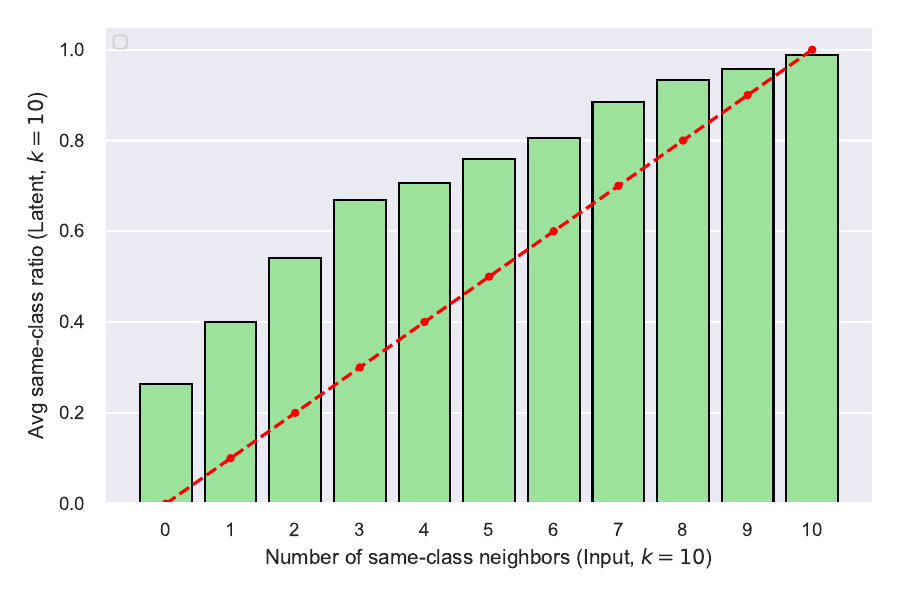}
        \caption{KAR}
    \end{subfigure}
    \hfill
    \begin{subfigure}{0.3\textwidth}
        \centering
        \includegraphics[width=\linewidth]{new_fig/exp8/optdigits_neighborhood_alignment.pdf}
        \caption{OPT}
    \end{subfigure}

    \vspace{0.5em}

    \begin{subfigure}{0.3\textwidth}
        \centering
        \includegraphics[width=\linewidth]{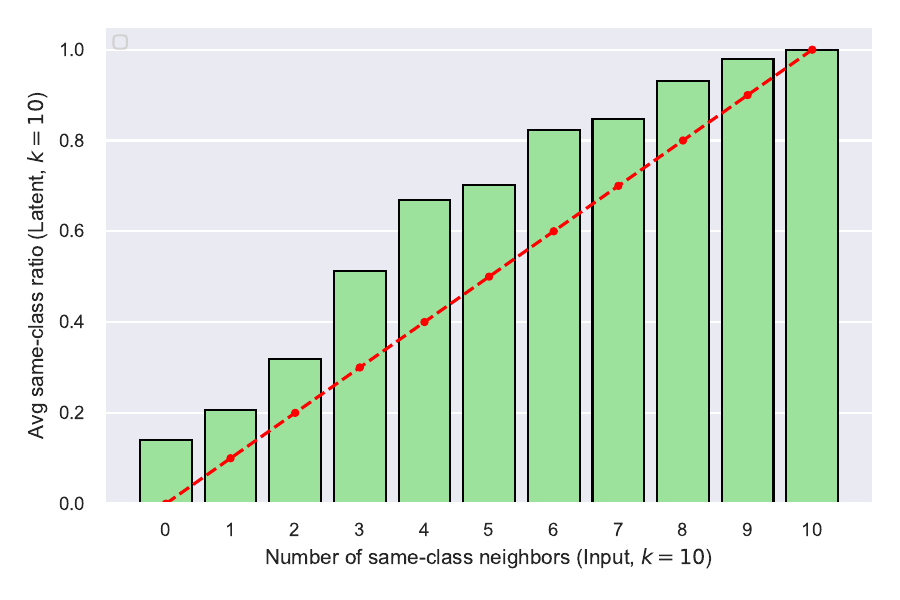}
        \caption{PIX}
    \end{subfigure}
    \hfil
    \begin{subfigure}{0.3\textwidth}
        \centering
        \includegraphics[width=\linewidth]{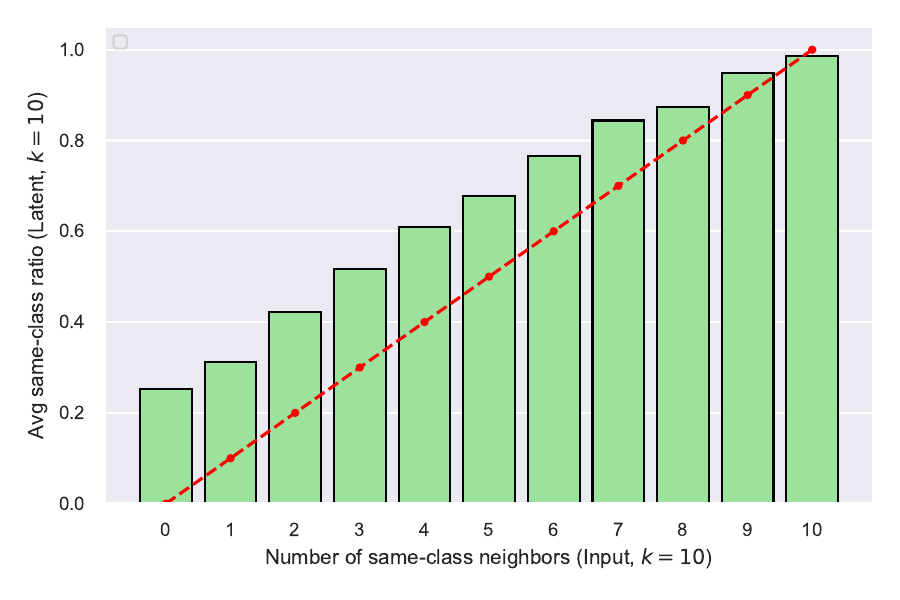}
        \caption{SEM}
    \end{subfigure}

    \caption{Comparison of label consistency between the baseline and the \our{} representation space. The Y-axis represents the number of 10-nearest neighbors sharing the same class label for each data point, grouped by 10-nearest neighbors in input space.}
        \label{fig:10nn_consistency_appendix}
\end{figure*}

\newcommand{\std}[1]{{\scriptsize $\pm #1$}}
\newcolumntype{C}{S[table-format=3.3, table-align-text-center]}

\begin{table*}
\centering
\caption{
Evaluation in terms of 1-, 5-, and 10-shot classification accuracy with standard deviations for 3 best-performing methods: STUNT~\cite{nam2023stunt} D2R2~\cite{liu2024d2r2}, and \our{}.
We put the best method in bold if its performance exceeds the second-best by more than one standard deviation of its result. This occurs in 18 out of 24 instances, with \our{} achieving the top result in 14 of those cases.
}

\label{tab:fewshot_main_results_with_std}
\setlength{\tabcolsep}{2pt}
\resizebox{\linewidth}{!}{
\begin{tabular}{l 
c c c c c c c c
}
\toprule
\textbf{Method} &
\multicolumn{1}{c}{CMC} &
\multicolumn{1}{c}{DIA} &
\multicolumn{1}{c}{DNA} &
\multicolumn{1}{c}{INC} &
\multicolumn{1}{c}{KAR} &
\multicolumn{1}{c}{OPT} &
\multicolumn{1}{c}{PIX} &
\multicolumn{1}{c}{SEM} \\
\midrule
\multicolumn{9}{c}{\textbf{\#shot=1}} \\
\midrule
STUNT & 37.10 \std{2.98} & 61.08 \std{4.47} & 66.20 \std{6.26} & 63.52 \std{4.75} & 71.20 \std{3.32} & 76.94 \std{2.28} & 79.05  \std{3.63} & 55.91 \std{3.65} \\ %
D2R2-c & \best{40.81} \std{0.41} & 60.10 \std{2.11} & 61.29 \std{2.16} & \best{72.85} \std{2.03} & 61.45 \std{1.78} & 77.41 \std{1.56} & 61.45 \std{1.15} & 34.26 \std{1.63} \\ %
\our{} (our) & 36.76 \std{0.72} & {61.14 \std{1.15}} & {66.79 \std{2.62}} & 62.89 \std{1.87} & \best{76.40 \std{1.01}} & \best{78.94 \std{0.88}} & \best{83.06 \std{1.20}} & \best{61.11 \std{1.07}} \\ %
\midrule
\multicolumn{9}{c}{\textbf{\#shot=5}} \\
\midrule
STUNT & 40.40 \std{3.55} & 69.88 \std{7.87} & 79.18 \std{6.54} & 72.69 \std{4.31} & 85.45 \std{2.11} & 88.42 \std{1.97} & 89.08 \std{2.06} & 71.54 \std{3.00} \\ %
D2R2-c & \best{43.39 \std{0.37}} & 68.69 \std{1.63} & \best{81.39 \std{1.38}} & {73.34 \std{1.91}} & 79.49 \std{0.95} & 87.12 \std{0.91} & 82.22 \std{1.93} & 60.16 \std{1.42} \\ %
\our{} (our) & 42.85 \std{0.62} & {69.54 \std{0.71}} & 79.86 \std{2.45} & 71.28 \std{0.65} & \best{87.59 \std{0.47} } & \best{90.11 \std{0.45} } & \best{91.88 \std{0.88}} & \best{79.41 \std{0.77} } \\ %
\midrule
\multicolumn{9}{c}{\textbf{\#shot=10}} \\
\midrule
STUNT & 42.01 \std{5.18} & 72.82 \std{4.17} & 80.96 \std{4.92} & 74.08 \std{3.33} & 86.95 \std{2.49} & 89.91 \std{1.40} & 89.98 \std{2.44} & 74.74 \std{3.08} \\ %
D2R2-c & 38.41 \std {0.58}  & 71.94 \std {1.61}  & 81.46 \std {1.89} & {74.84} \std {1.97}  & 86.19 \std {0.75}  & 91.84 \std {0.74}  & 84.93 \std {1.17} & 69.56 \std {1.28} \\ %
\our{} (our) & \best{46.30 \std{0.62} } & \best{73.61} \std{0.43} & {83.59 \std{2.21}} & 72.68  \std{0.52} & \best{90.88 \std{0.38}} & \best{92.62 \std{0.42}} & \best{93.88 \std{0.43}} & \best{84.11 \std{0.74}} \\ %
\bottomrule
\end{tabular}
}
\end{table*}

\end{document}